\documentclass{article}

\usepackage[arxiv]{optional}

\usepackage{etoolbox}
\newbool{istwocolumn}
\newbool{usebiblatex}
\booltrue{usebiblatex}

\opt{arxiv}{
    \usepackage{arxiv}
    \boolfalse{istwocolumn}
}
\opt{aistats}{
    \usepackage{aistats2023}
    \booltrue{istwocolumn}
}

\usepackage{etoolbox}
\providebool{loadgeometry} %
\providebool{altermargins} %
\providebool{usebiblatex}  %

\usepackage[pdftex]{graphicx}
\usepackage{amssymb}

\usepackage{scalefnt,letltxmacro}
\LetLtxMacro{\oldtextsc}{\textsc}
\renewcommand{\textsc}[1]{\oldtextsc{\scalefont{1.10}#1}}
\usepackage[ttdefault=true]{AnonymousPro}

\usepackage{amsmath}
\usepackage{amsmath,amsthm, mathtools}
\usepackage{amscd}
\usepackage{bbm}
\usepackage{commath}      %
\usepackage{dsfont}       %
\usepackage{upgreek}      %
\usepackage{stmaryrd}     %
\usepackage{siunitx}      %
\usepackage{xspace}
\usepackage{float}        %
\usepackage{xfrac}
\usepackage[nocomma]{optidef}

\ifbool{loadgeometry}{
	\ifbool{altermargins}{
		\usepackage[
		  paper  = letterpaper,
		  left   = 1.0in,
		  right  = 1.0in,
		  top    = 1.0in,
		  bottom = 1.0in,
		  ]{geometry}
	}{
		\usepackage{geometry} 	%
	}
}

\usepackage[english]{babel}
\usepackage[parfill]{parskip}
\usepackage{afterpage}
\usepackage{enumitem}
\usepackage{framed}
\usepackage{xspace}
\usepackage{xparse}%
\usepackage{csquotes} %
\usepackage{adjustbox}

\usepackage[usenames,dvipsnames]{xcolor}
\definecolor{shadecolor}{gray}{0.9}

\usepackage{booktabs, array}

\usepackage[labelfont=bf]{caption}
\usepackage[format=hang]{subcaption}
\usepackage{wrapfig}
\usepackage{graphbox} %

\usepackage[algoruled]{algorithm2e}
\usepackage{listings}
\usepackage{fancyvrb}
\fvset{fontsize=\small}

\usepackage
[acronym,smallcaps,nowarn,section,nogroupskip,nonumberlist]{glossaries}
\glsdisablehyper{}

\usepackage[colorlinks]{hyperref} %
\hypersetup{citecolor=MidnightBlue}
\hypersetup{linkcolor=BrickRed}
\hypersetup{urlcolor=MidnightBlue}
\usepackage[nameinlink]{cleveref}
\creflabelformat{equation}{#2\textup{#1}#3}  %
\crefformat{section}{#2§#1#3}
\Crefformat{appendix}{\S#2#1#3}

\crefname{section}{§}{§§}
\Crefname{section}{§}{§§}
\usepackage{url}

\lstdefinestyle{mystyle}{
    commentstyle=\color{OliveGreen},
    numberstyle=\tiny\color{black!60},
    stringstyle=\color{BrickRed},
    basicstyle=\ttfamily\scriptsize,
    breakatwhitespace=false,
    breaklines=true,
    captionpos=b,
    keepspaces=true,
    numbers=none,
    numbersep=5pt,
    showspaces=false,
    showstringspaces=false,
    showtabs=false,
    tabsize=2
}
\ifbool{usebiblatex}{
	\usepackage[
	    backend=biber,
		citestyle=authoryear-comp,
		uniquename=init,
	    defernumbers=true,       %
		giveninits=true,         %
		natbib=true,
		isbn=false,
		uniquelist=false, %
		minnames=1,
    	maxnames=2, %
    	maxbibnames=4, %
		eprint=true,
		url=false,
		doi=false,
	]{biblatex}
	\setlength\bibitemsep{0.2\baselineskip}
}{
	\usepackage{natbib}
}

\newcommand{\q}[1]{\red{{\sf Q | #1}}}

\usepackage{xargs}                      %
\usepackage[colorinlistoftodos,prependcaption,textsize=tiny]{todonotes}
\newcommandx{\unsure}[2][1=]{\todo[linecolor=red,backgroundcolor=red!25,bordercolor=red,#1]{#2}}
\newcommandx{\change}[2][1=]{\todo[linecolor=blue,backgroundcolor=blue!25,bordercolor=blue,#1]{#2}}
\newcommandx{\info}[2][1=]{\todo[linecolor=OliveGreen,backgroundcolor=OliveGreen!25,bordercolor=OliveGreen,#1]{#2}}
\newcommandx{\improvement}[2][1=]{\todo[linecolor=Plum,backgroundcolor=Plum!25,bordercolor=Plum,#1]{#2}}
\newcommandx{\thiswillnotshow}[2][1=]{\todo[disable,#1]{#2}}

\newtheorem{theorem}{Theorem}[section]

\newtheorem{proposition}[theorem]{Proposition}

\newtheorem{lemma}[theorem]{Lemma}
\newtheorem{lemma*}{Lemma}
\newtheorem{corollary}[theorem]{Corollary}

\newtheorem{remark}[theorem]{Remark}

\newtheorem{assumption}{\textbf{Assumption}}

\def\C{\mathbb{C}}
\def\K{\mathbb{K}}
\def\N{\mathbb{N}}
\def\Q{\mathbb{Q}}

\def\A{\mathbf{A}}
\def\B{\mathbf{B}}
\def\D{\mathbf{D}}
\def\F{\mathbf{F}}
\def\G{\mathbf{G}}
\def\I{\mathbf{I}}
\def\K{\mathbf{K}}
\def\L{\mathbf{L}}
\def\P{\mathbf{P}}
\def\U{\mathbf{U}}
\def\V{\mathbf{V}}
\def\X{\mathbf{X}}
\def\Y{\mathbf{Y}}

\def\a{\mathbf{a}}
\def\b{\mathbf{b}}
\def\c{\mathbf{c}}
\def\d{\mathbf{d}}
\def\e{\mathbf{e}}
\def\f{\mathbf{f}}
\def\g{\mathbf{g}}
\def\p{\mathbf{p}}
\def\q{\mathbf{q}}
\def\r{\mathbf{r}}
\def\s{\mathbf{s}}
\def\u{\mathbf{u}}
\def\v{\mathbf{v}}
\def\w{\mathbf{w}}
\def\x{\mathbf{x}}
\def\y{\mathbf{y}}
\def\z{\mathbf{z}}

\def\0{\mathbf{0}}

\def\cP{\mathcal{P}}

\def\cX{\mathcal{X}}
\def\cY{\mathcal{Y}}

\DeclarePairedDelimiterX{\inner}[2]{\langle}{\rangle}{#1, #2} %

\DeclareMathOperator*{\Exp}{\mathbb{E}}

\DeclareMathOperator{\iid}{\stackrel{\mathclap{\tiny\mbox{i.i.d.}}}{\sim}}

\def\H{\text{H}}             %

\newcommand{\tr}{\mathrm{tr}}

\def\ones{\mathds{1}}

\makeatletter
\renewcommand*\env@matrix[1][*\c@MaxMatrixCols c]{%
  \hskip -\arraycolsep
  \let\@ifnextchar\new@ifnextchar
  \array{#1}}
\makeatother

\newenvironment{itemize*}%
  {\begin{itemize}%
  \vspace{-0.5cm}
    \setlength{\itemsep}{0pt}%
    \setlength{\parskip}{0pt}}%
  {\end{itemize}}
\newenvironment{enumerate*}%
{\begin{enumerate}
    \vspace{-0.5cm}
    \setlength{\itemsep}{0pt}%
    \setlength{\parskip}{0pt}}%
  {\end{enumerate}}

\input{math_commands}

\opt{arxiv}{
    \usepackage[utf8]{inputenc} %
    \usepackage[T1]{fontenc}    %
    \usepackage{amsfonts}       %
    \usepackage{nicefrac}       %
    \usepackage{microtype}      %
    \usepackage{doi}
}

\opt{arxiv}{
    \title{Budget-Constrained Bounds for \\Mini-Batch Estimation of Optimal Transport}
    \author{
    David Alvarez-Melis \\
    Microsoft Research \\
	\texttt{daalvare@microsoft.com} \\
	\And
    Nicol\`o Fusi \\
    Microsoft Research \\
	\texttt{fusi@microsoft.com} \\
	\AND
    Lester Mackey \\
    Microsoft Research \\
	\texttt{lmackey@microsoft.com} \\
	\And
    Tal Wagner\thanks{Work done prior to joining Amazon.} \\
    Amazon \\
	\texttt{tal.wagner@gmail.com} \\	
    }
    \date{}
    \renewcommand{\headeright}{}
    \renewcommand{\undertitle}{}
    \renewcommand{\shorttitle}{Budget-Constrained Mini-Batch Optimal Transport Bounds}

    \hypersetup{
    pdftitle={A template for the arxiv style},
    pdfsubject={q-bio.NC, q-bio.QM},
    pdfauthor={David S.~Hippocampus, Elias D.~Striatum},
    pdfkeywords={First keyword, Second keyword, More},
    }
}

\begin{document}
\opt{aistats}{
    \twocolumn[
    \aistatstitle{Budget-Constrained Bounds for Mini-Batch Estimation of Optimal Transport}
    \aistatsauthor{ Author 1 \And Author 2 \And  Author 3 }
    \aistatsaddress{ Institution 1 \And  Institution 2 \And Institution 3 } ]
}
\opt{arxiv}{
    \maketitle
}

\begin{abstract}
Optimal Transport (OT) is a fundamental tool for comparing probability distributions, but its exact computation remains prohibitive for large datasets. In this work, we introduce novel families of upper and lower bounds for the OT problem constructed by aggregating solutions of mini-batch OT problems. 
The upper bound family contains traditional mini-batch averaging at one extreme and a tight bound found by optimal coupling of mini-batches at the other. 
In between these extremes, we propose various methods to construct bounds based on a fixed computational budget. Through various experiments, we explore the trade-off between computational budget and bound tightness and show the usefulness of these bounds in computer vision applications.
\end{abstract}
\section{Introduction}

Optimal Transport (OT) distances, in particular the Wasserstein distance, have become a popular tool in machine learning for tasks ranging from domain adaptation \citep{courty2017optimal} to generative modeling \citep{genevay2018learning, salimans2018improving}. From among its many desirable properties, we highlight that OT provides a principled and general approach to lift a metric between samples into one between distributions, is underpinned by a mature theory \citep{villani2003topics, villani2008optimal}, and has a well-understood sample complexity  \citep{genevay2019sample, mena2019statistical}. \looseness=-1

Historically, a primary barrier to the wider adoption of OT in machine learning and other data-intensive fields has been its computational cost. In the classic formulation by \citet{kantorovich1942translocation}, the discrete OT problem is a linear programming (LP) problem with cubic complexity and quadratic memory footprint, prohibitive for all but the smallest datasets. Over the past decade, there has been considerable progress towards scaling up the computation of OT, typically by settling for an approximate solution by solving an entropy-regularized problem instead \citep{cuturi2013sinkhorn}. Despite is convenience, this approximation is not always desirable, as it introduces a statistical bias in the problem \citep{chizat2020faster}, yields non-sparse solutions \citep{blondel2018smooth}, and is infamously sensitive to the choice of regularization strength parameter.

But even with entropy-regularized approximation, OT on datasets of machine learning scale remains elusive. Although highly optimized off-the-shelf solvers \citep{flamary2021pot, feydy2019interpolating} have made it possible to solve much larger problems, OT computations on the full MNIST \citep{lecun2010mnist}---the archetypal \textit{toy} machine learning dataset---remains challenging for both entropy-regularized and exact OT on most personal computers, largely due to memory bottlenecks. Scaling up to even larger (but, by machine learning standards, still `benchmark'-sized) datasets such as ImageNet \citep{deng2009imagenet} ($\sim$1M samples of $\sim$50K dimension) is therefore currently infeasible with the standard OT formulation.

A common approach to circumvent this issue in practice is mini-batch estimation: computing OT on smaller subsamples (\textit{mini-batches}) of the data and then averaging their values. This has been particularly exploited in applications of OT to generative modeling \citep{genevay2018learning, salimans2018improving}. Although this was originally done heuristically, recent work has started to investigate the properties of this type of estimator \citep{fatras2021minibatch, fatras2020learning}. For example, it can be shown that this estimator is biased; in fact, it is an upper bound on the exact (full-sample) OT distance and often a loose one. A natural question is whether other,  significantly tighter mini-batch based bounds exist. 

In this work, we introduce new families of upper and lower bounds for the discrete OT problem that rely solely on solutions of mini-batch problems. The key idea behind them is to conceptually break up the original (full-sample) problem into a blockwise-grid of mini-batch problems and find weighted combinations of their solutions that are feasible for the original problem (see Figure~\ref{fig:main_digram}). These bounds can be tightened by optimizing the total cost of these linear combinations, which itself corresponds to an OT problem where the mini-batches play the role of samples. We show that the tightest bound in this class (i.e., the one corresponding to the optimal batch-to-batch coupling) is obtained by solving all the pair-wise mini-batch subproblems \citep{nguyen2022transportation}. This method, however, has quadratic dependence on the number of batches, and, as we show here, the complexity of computing this bound is as a bad as that of solving the entire problem when not parallelizing and thus is often still infeasible in practice.

In response, we propose to approximate this best-in-class bound using the solution of only a subset of the mini-batch problems with size determined by a pre-specified budget. We investigate various approaches to select the subset of problems to solve, ranging from greedy heuristic methods to provable approximate OT methods. The resulting bounds interpolate between the usual mini-batch averaging bound (which has a linear dependence on the number of batch problems solved) and the `all pair-wise problem' (quadratic) solution, providing a simple way to trade off computational cost for estimation accuracy (Fig.~\ref{fig:couplings}). 

We investigate the empirical behaviour of all of these bounds in a series of experimental evaluations on computer vision datasets. The results confirm that these estimators reliably trade-off computation for accuracy and that they provide reasonable approximations to the best-in-class bound even in low-budget settings. In particular, our experiments with two-sample testing show that these estimators yield tests with negligible decrease in power compared to much more expensive estimators.

\ifbool{istwocolumn}{
\begin{figure*}[t!]
    \centering
    \includegraphics[width=\textwidth]{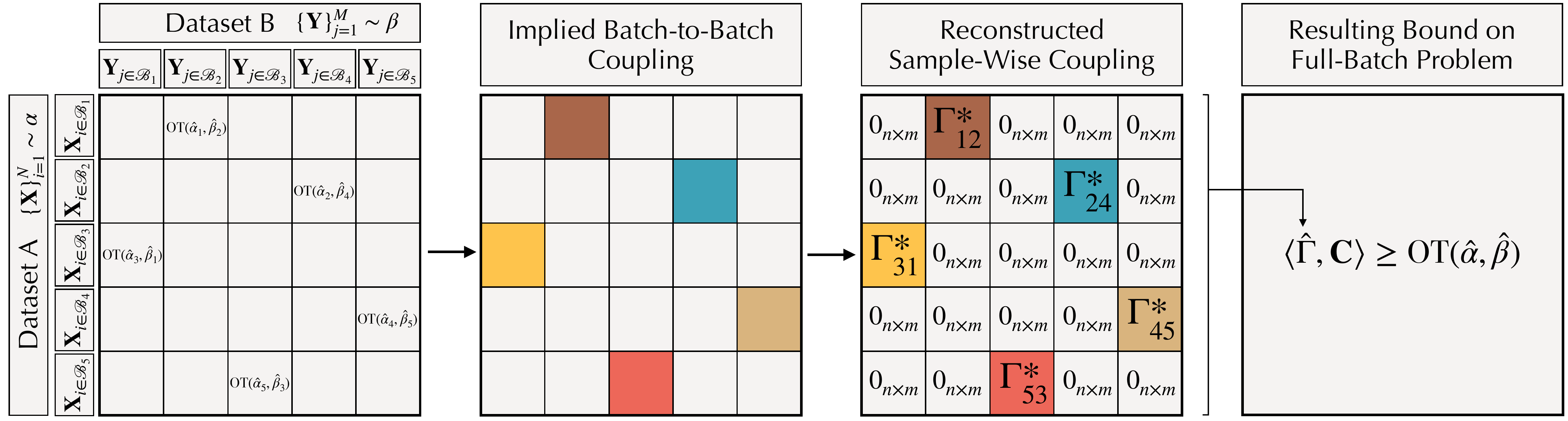}
    \caption{\textbf{Bounds via block solutions}. For two datasets split into mini-batches, any choice of pairs of mini-batches to compare implies a (batch-wise) coupling between them. Using the solutions of these batch-to-batch problems, we can construct a feasible solution to the full-batch problem, thus obtaining an upper bound for its optimal solution. Here we explore strategies to tighten this bound by optimizing the construction of such solutions for a given computational budget.
    }
    \label{fig:main_digram}
\end{figure*}
}{
\begin{figure}[t]
    \centering
    \includegraphics[width=\textwidth]{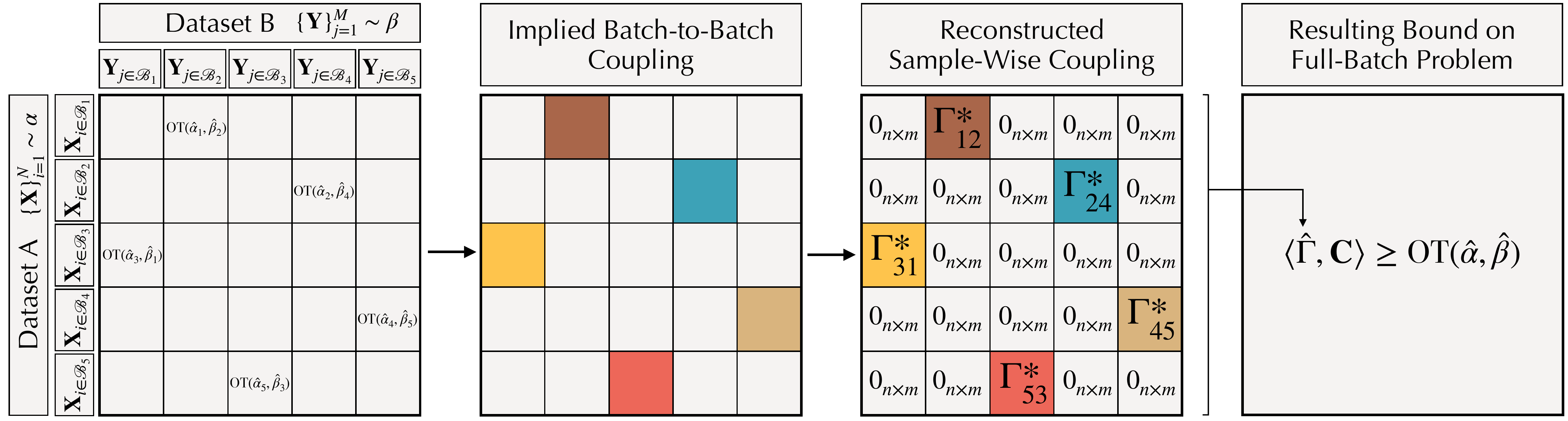}
    \caption{\textbf{Bounds via block solutions}. For two datasets split into mini-batches, any choice of pairs of mini-batches to compare implies a (batch-wise) coupling between them. Using the solutions of these batch-to-batch problems, we can construct a feasible solution to the full-batch problem, thus obtaining an upper bound for its optimal solution. Here we explore strategies to tighten this bound by optimizing the construction of such solutions for a given computational budget.
    }
    \label{fig:main_digram}
\end{figure}
}
\section{Related Work}\label{sec:related}

\paragraph{Mini-Batch Optimal Transport.} Estimation of OT distances through mini-batch computation was first proposed by \citet{genevay2019sample, salimans2018improving} in the context of using entropy-regularized OT as a loss function for generative modeling. \citet{fatras2020learning, fatras2021minibatch, fatras2021unbalanced} have thoroughly studied the properties of the usual naive averaging mini-batch estimator. Recent work by \citet{nguyen2022transportation} proposes an estimator to the OT problem that is similar to the first upper bound (that requires solving \textit{all} subproblems) proposed here. Although both rely on a hierarchical characterization of the OT problem, our work differs in that we study this approximate solution as an upper bound (instead of a generic estimator), we provide accompanying lower bounds, and---crucially---we propose budget-constrained bounds that avoid the need to solve all mini-batch problems. 

\paragraph{Hierarchical, anchor, and low-rank OT.}
Nested or hierarchical OT distances have also been proposed in other contexts, such as topic modeling \citep{yurochkin2019hierarchical} and medical imaging \citep{yeaton2022hierarchical}. These works typically seek to model existing domain-specific hierarchical structures, rather than imposing them for computational reasons, as we do here. Also related are anchor-based OT distances \citep{sato2020fast, lin2021making}, which impose a hierarchical structure to make the solution of the OT problem more robust and interpretable. A different but equally active area of research seeks to improve the computational complexity of the OT problem through low-rank approximations for the cost matrix \citep{altschuler2020polynomial}, transport plan \citep{forrow2019statistical, scetbon2021low-rank}, or both \citep{scetbon2021linear-time}.

\paragraph{Efficient approximate OT.} When the ground metric is embedded into the $\ell_p$-distance with $1\leq p\leq2$ in $\R^d$, a fruitful line of work has focused on fast multiplicative approximation algorithms for OT using probabilistic tree embeddings, including Quadtree \citep{charikar2002similarity,indyk2003fast,andoni2008earth}, Flowtree \citep{backurs2020scalable,chen2022new}, and the Tree-Sliced Wasserstein distance \citep{le2019tree,le2021flow}, which can be seen as a generalization of the Sliced Wasserstein distance \citep{rabin2011wasserstein,bonneel2015sliced,kolouri2016sliced,carriere2017sliced}. These methods provably approximate the OT cost up to (poly)logarithmic factors, while running in time nearly linear in the total support size of the input measures and in the dimensionality $d$. Thus, they are efficient even in the high-dimensional regime.

\paragraph{Other bounds on OT.} Although upper bounds to the OT problem can be easily constructed, efficiently computable ones are scarce. Among those, there are coupling-based methods \citep{biswas2021bounding} and variational methods \citep{huggings2020variational}. Lower bounds (even non-computable ones) are harder to construct. Perhaps the best-known one is by \citet[Thm.~2.1]{gelbrich1990formula}.

\section{Computing Optimal Transport}
We consider measurable Polish spaces $\cX$ and $\cY$, where typically $\cX,\cY \subset \sR^d$, and denote by $\cP(\cX), \cP(\cY)$ the spaces of Borel probability measures defined on them. For samples $\rmX = (X_1, \dots, X_n) \iid \alpha \in \cP(\cX)$ and $\rmY=(Y_1, \dots, Y_m) \iid \beta \in \cP(\cY)$, we denote by $\alpha_n$ and $\beta_m$ the empirical distributions supported on $\rmX$ and $\rmY$ respectively. 
We denote by $m(\alpha)= \int_{\cX} \dif\alpha$ the total mass of a measure and by $\alpha \otimes \beta$ the product measure. When $n>1$, $\alpha^{\otimes n} \eqdef \alpha \otimes \cdots \otimes \alpha$. Thus, an i.i.d.~sample $\X$ of size $n$ from $\alpha$ has law $\alpha^{\otimes n}$. %
For $\a\!\in\!\R^n, \b\!\in\!\R^ m$, $\a \oplus \b$ is the matrix with entries $(\a \oplus \b)_{ij} = a_i + b_j$, i.e., $\a \oplus \b = \a \ones_m^\top + \ones_n \b^\top$. Finally, $\Delta_n$ denotes the probability $n$-simplex and $\llbracket n \rrbracket \eqdef \{1, \dots, n\}$. 

\subsection{Finite-sample estimation of OT distances}

Consider two probability measures $\alpha \in \mathcal{P}(\cX), \beta \in \cP(\cY)$. The Kantorovich formulation of Optimal Transport allows us to compare them in terms of
\begin{equation}\label{eq:ot_continuous}
    \OT(\alpha, \beta) = \inf_{\pi \in \Pi(\alpha, \beta)} \iint_{\cX \times \cY} c(x,y) \dif \pi(x,y)    
\end{equation}
where $\Pi(\alpha, \beta)$ is the set of couplings (\textit{transportation plans}) between $\alpha$ and $\beta$.
Formally, $\Pi(\alpha, \beta) \eqdef \{\pi \in \cP(\cX\!\times\!\cY) \suchthat P_{x\sharp}\pi = \alpha,  P_{y\sharp}\pi=\beta \}$. In practice, the measures of interest are often discrete (e.g., histograms $\a \in \Delta_n, \b \in \Delta_m$) or are continuous but accessible only through finite samples, in which case we can express them as $\alpha_n = \frac{1}{n}\sum_{i=1}^n a_i\delta_{X_i}$, $\beta_m = \frac{1}{m}\sum_{j=1}^m b_j\delta_{Y_j}$, where $X_i$ and $Y_j$ are the support points and $\a \in \Delta_n, \b\in \Delta_m$ are their associated probability vectors. For the latter case (finite samples from continuous measures), we assume $(X_1, \dots, X_n) \iid \alpha$, $(Y_1, \dots, Y_m) \iid \beta$ and thus use uniform weights $\a = \tfrac{1}{n} \ones_n$, $\b = \tfrac{1}{m} \ones_m$. \looseness=-1

For discrete distributions, OT becomes a linear program,
\begin{equation*}%
    \!\!\!\OT(\alpha_n, \beta_m) = \!\!\min_{\P \in \U(\a, \b)} \!\langle \mathbf{C} , \P \rangle  = \!\min_{\P \in \U(\a, \b)} \! \sum_{i,j=1}^{n,m} C_{ij}P_{ij},
\end{equation*}
where $C_{ij} = c(x_i, y_j)$, $P_{ij} = \pi(x_i, y_j)$, and
\begin{equation*}
    \U(\a, \b) \eqdef \bigl \{  \P \in \sR^{n\times m}_+ \st \P\ones_m=\a, \P^{\top}\ones_n = \b \bigr \}.
\end{equation*}

 This problem has an equivalent dual formulation,
\begin{equation}\label{eq:ot_primal}
   \!\!\OT(\alpha_n, \beta_m) = \sup_{(\f, \g) \in \mathcal{R}(\mathbf{C})}  \langle \f, \a \rangle +  \langle \g, \b \rangle,
\end{equation}
where $\f$ and $\g$, called the Kantorovich potentials, are taken in the set $\mathcal{R}(\mathbf{C})=\{ (\f,\g) \in \R^{n\times m} \st \f \oplus \g \leq \mathbf{C}\}$.  

\citet{kosowsky1994invisible,galichon2009matching,cuturi2013sinkhorn} introduced a regularized version of \plaineqref{eq:ot_primal}:
\begin{equation}\label{eq:ot_primal_entreg}
    \OTe(\alpha_n, \beta_m) \eqdef  \min_{\P \in \U(\a, \b)} \langle \mathbf{C} , \P \rangle + \varepsilon H(\P),
\end{equation}
which can be solved much more efficiently, e.g., using the Sinkhorn algorithm \citep{sinkhorn1964relationship}. Since $\OTe(\alpha, \alpha)=0$ is not guaranteed, a `debiased' version this problem, known as the Sinkhorn Divergence \citep{feydy2019interpolating, genevay2019sample} is often used instead:
\begin{equation*}%
    \SDe(\alpha, \beta) \eqdef \OTe(\alpha, \beta) - \tfrac{1}{2} \bigl(\OTe(\alpha, \alpha) +\OTe(\beta, \beta) \bigr).
\end{equation*}
Henceforth, we assume a choice of OT `kernel function' (one of $\OT, \OTe$, or $\SDe$), but denote it indistinctly by $\OT(\cdot,\cdot)$ for notational simplicity.

\subsection{Mini-batch partitioning of datasets}\label{sec:partitioning}
Consider a pair of datasets $\X\in \R^{N\times d}$ and $\Y\in \R^{M\times d}$ with associated empirical distributions $\alpha_N, \beta_M$. If $N$ and $M$ are large, computing $\OT(\alpha_N, \beta_M)$ directly might be prohibitive. A common strategy is to instead solve smaller problems using \textit{mini-batches} from $\X$ and $\Y$ to estimate this quantity \citep{genevay2018learning}. Given a predefined batch size $n$ (and assuming for simplicity that $N/n=k$ is integer), we can view $\X$ as being sampled batch-wise as $\X = \bigl ( (X_1^1, \dots, X_n^1), \dots,  (X^k_1, \dots, X^k_n) \bigr) = (\X^1, \dots, \X^k)$, where  $\X^s \sim \alpha^{\otimes n}$. Analogously, we write $\Y = (\Y^1, \dots, \Y^k)$ with $\Y^t\sim \beta^{\otimes m}$. Although this implicitly assumes that the batches are taken as contiguous subsets of the original dataset (i.e., $X_i^s = X_{n(s-1)+i}$ ), as typically done in practice,  in general any partition of the samples gives rise to a meaningful set of mini-batches. Thus, we can generally consider mini-batches defined by lists of indices $\bigcup_{s=1}^k \mathrm{B}^s_{x}  = \llbracket N \rrbracket$ and $\bigcup_{t=1}^k \mathrm{B}^t_{y}  = \llbracket M \rrbracket$. For convenience, we define mappings $\sigma_x$ and $\sigma_y$ from mini-batch index to dataset index, whereby for $i \in \llbracket n \rrbracket, s \in \llbracket k \rrbracket$, $\sigma_x(i,s)= j \in \llbracket N \rrbracket$ means the $i$-th element in $\mathrm{B}^s_{x}$ has index $j$ in the full dataset (i.e., $X^s_i = X_{\sigma_x(s,i)}$) and analogously for $\sigma_y$. %
In addition, for simplicity we will assume $n$ and $m$ are chosen so that $N/n=M/m=k$.  \looseness=-1

Each mini-batch from $\X$ has two associated empirical measures: an unnormalized one $\tilde{\alpha}^{s}= \sum_{i\in  \mathrm{B}^s_{x} }\a_{i}\delta_{X_i}$ and a normalized one $\alpha^{s}= \sum_{i=1}^n\a_i^s\delta_{X_i^s}$, with $\a^s \in \Delta_n$. When the full discrete measure $\alpha_N$ has uniform weights $(\a_i = \tfrac{1}{N})$, we simply have $\a_i^s = \frac{N}{n}\a_{\sigma_x(i,s)} = \frac{1}{n}$. In general, we have $\a_i^s = \tfrac{1}{m(\tilde{\alpha}^{s})}\a_{\sigma_x(i,s)}$. For $\Y$, $\tilde{\beta}^{t}$ and $\beta^{t}$ are defined analogously. We also define the probability vectors of aggregated mini-batch masses $\tilde{\a},\tilde{\b} \in \Delta_k$ with entries $\tilde{\a}_s \eqdef m(\tilde{\alpha}^s) = \sum_{ i \in \mathrm{B}^s_{x}} \a_i$ and $\tilde{\b}_t \eqdef m(\tilde{\beta}^t) = \sum_{ j \in \mathrm{B}^t_{y}} \b_j$\footnote{In matrix form: $\tilde{\a} = (\mathbf{I}_k \otimes \ones_n^\top) \a$, and $\tilde{\b} = (\mathbf{I}_k \otimes \ones_m^\top) \b$.}. For the remainder of this work, we will make the following simplifying assumption, which holds trivially in the usual uniform-weight setting and can be imposed with judicious choice of mini-batch sizes in general:
\begin{assumption}[Uniform mass mini-batches]
  The grouping of samples into mini-batches is such that $m(\tilde{\alpha}^s) = m(\tilde{\beta}^t)$ for every $s,t \in \llbracket k \rrbracket$ (i.e., $\tilde{\a} = \tilde{\b} = \tfrac{1}{k}\ones_k$). 
\end{assumption}

The grouping into mini-batches induces a partition in the discrete optimal transport problem between $\alpha_N$ and $\beta_M$ as well. The full-sample cost matrix $\mathbf{C} \in \R^{N\times M}$, where $\mathbf{C}_{ij} = c(X_i, Y_j)$, inherits a $k$-by-$k$ block structure whereby the $s,t$-block is $\mathbf{C}^{s,t} \in \R^{n\times m}$, with $\mathbf{C}^{s,t}_{ij}= c(X^s_i, X^t_j)$. Analogously, the same block structure partitions $\P \in \R^{N\times M}_{+}$ into $k^2$ submatrices $\P^{s,t} \in \R^{n\times m}_{+}$ with entries $\P^{s,t}_{ij} = \P(X^s_i, X^t_j)$. Our goal is to estimate $\OT(\alpha_N, \beta_M)$ using the solutions of the subproblems $\OT(\alpha^{s},\beta^{t})$. Over the next two sections, we present methods to obtain upper and lower bounds for the former using the latter.

%

%
%
%

%

%

\section{Upper Bounds via Primal Solutions}\label{sec:ub}

In this section, we introduce a family of upper bounds to the OT problem constructed using solutions to mini-batch subproblems. We discuss lower bounds in Appendix~\ref{sec:lower_bounds}. 

\subsection{Bounding through mini-batch coupling}
Recall the primal form of the discrete OT problem between the empirical distributions $\alpha_N$ and $\beta_M$,
\begin{equation}\label{eq:primal_prob_full}
	\OT(\alpha_N, \beta_M) = \min_{\P \in \U(\a, \b)} \langle \P, \mathbf{C} \rangle.	\tag{FP}
\end{equation}
The partitioning described in the previous section gives rise to $k^2$ mini-batch problems, one for each $(s,t)$ pair:
\begin{equation}\label{eq:primal_prob_batch}
	\medspace \OT(\alpha^{s},\beta^{t}) = \min_{\P \in \U(\a^{s},\b^{t})} \langle \P, \mathbf{C}^{st} \rangle \quad s,t\in \llbracket k \rrbracket.	\tag{BP}
\end{equation}
Let $\P^{st}_{*}$ be the optimal solutions of these subproblems and $d_{st} \eqdef \OT(\alpha^{s},\beta^{t})$ their value. Our goal is use these to estimate $\OT(\alpha_N, \beta_M)$. To this end, it will be useful to rewrite problem \plaineqref{eq:primal_prob_full} using the block-structure induced by the mini-batch partitioning: 
\begin{equation}\label{eq:primal_prob_full_blockified}
\begin{gathered}
\OT(\alpha_N, \beta_M) = \min_{\P} \sum_{s,t=1}^k \langle \P^{st}, \mathbf{C}^{st} \rangle \quad \text{subject to:} \tag{FP'} \\[0ex] \forall s: 	\sum_{t=1}^K   (\P^{st}) \ones_m \! = \a^{s} ,  \  \forall t:  \sum_{s=1}^K  (\P^{st})^{\top} \ones_n \!= \b^{t}
\end{gathered}
\end{equation}
Crucially, note that while \plaineqref{eq:primal_prob_batch} is a collection of decoupled problems, \plaineqref{eq:primal_prob_full_blockified} is a single (coupled) problem.

The OT problem can be easily upper-bounded by constructing feasible solutions to its primal (minimization) problem. In our setting, a natural idea is to construct a feasible solution $\P$ for \plaineqref{eq:primal_prob_full} using the batch-wise optimal couplings $\P^{st}_{*}$ obtained by solving the sub-problems \plaineqref{eq:primal_prob_batch}. Theorem~\ref{thm:blockwise_linear_primal} provides a simple way to do so using scalar multiples of the mini-batch problem solutions. \looseness=-1

\begin{theorem}\label{thm:blockwise_linear_primal}
    Assume uniform-weight partitions, and let $\{\P^{(st)}_*\}_{s,t=1}^{k,k}$ be optimal solutions to the batch primal problems \plaineqref{eq:primal_prob_batch}. Let $\tilde{\P} \in \R^{N\times M}$ be defined block-wise\footnote{Equivalently, $\tilde{\P} \eqdef \mathbf{W} \ast \P$, where $\ast$ is the Khatri-Rao product \citep{khatri1968solutions}.} as  
    \begin{equation}\label{eq:blockwise_weighted_coupling}
        [\tilde{\P}_{ij}]_{i \in \mathrm{B}^s_x, j\in \mathrm{B}^t_y} \eqdef \omega_{st} \P^{st}_* \qquad  \forall s,t \in \llbracket k \rrbracket
    \end{equation}
    for some scalar values $\omega_{st}$, and let $\mathbf{W}$ be the $k$-by-$k$ matrix consisting of these values. If $\mathbf{W} \in \mathbf{U}(\tilde{\a}, \tilde{\b})$, then $\tilde{\P}$ is feasible for \plaineqref{eq:primal_prob_full} with objective value
    \[   \langle \tilde{\P}, \mathbf{C} \rangle = \langle \mathbf{W}, \D \rangle =  \sum_{s,t}^k \omega_{st} \OT(\alpha^s,\beta^t),   \]
    where the entries of $\D$ are $d_{st} \triangleq \OT(\alpha^s,\beta^t)$, the optimal values to the mini-batch problems. 
\end{theorem}

\begin{figure*}
     \centering
     \begin{subfigure}[b]{0.32\textwidth}
         \centering
        \includegraphics[width=\textwidth, trim= {0 0.2cm 1.8cm 1cm}, clip]{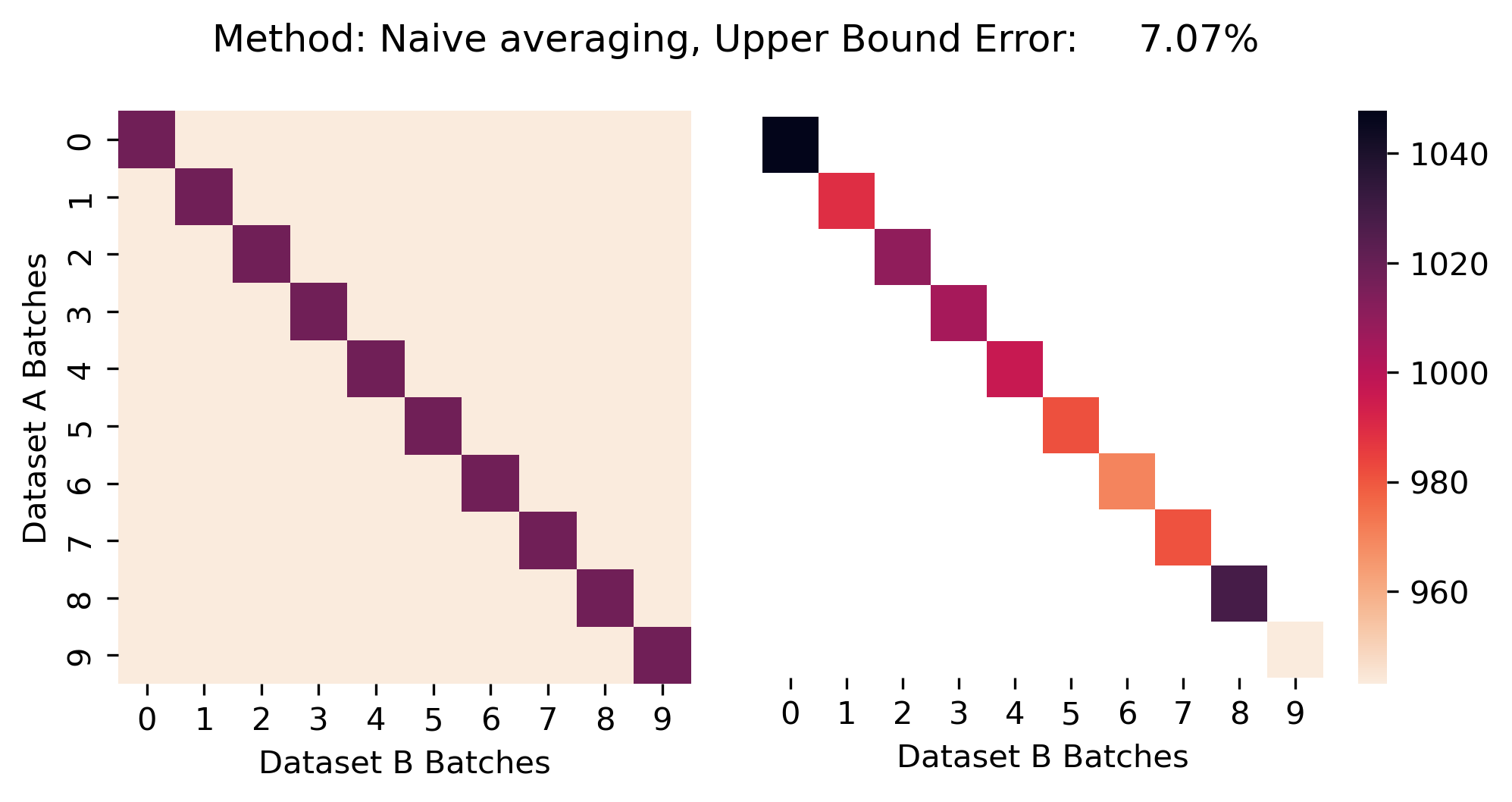}%
         \caption{Naive random batch matching}
         \label{fig:couplings_naive}
     \end{subfigure}
     \hfill
     \begin{subfigure}[b]{0.32\textwidth}
         \centering
        \includegraphics[width=\textwidth, trim= {0 0.2cm 1.8cm 1cm}, clip]{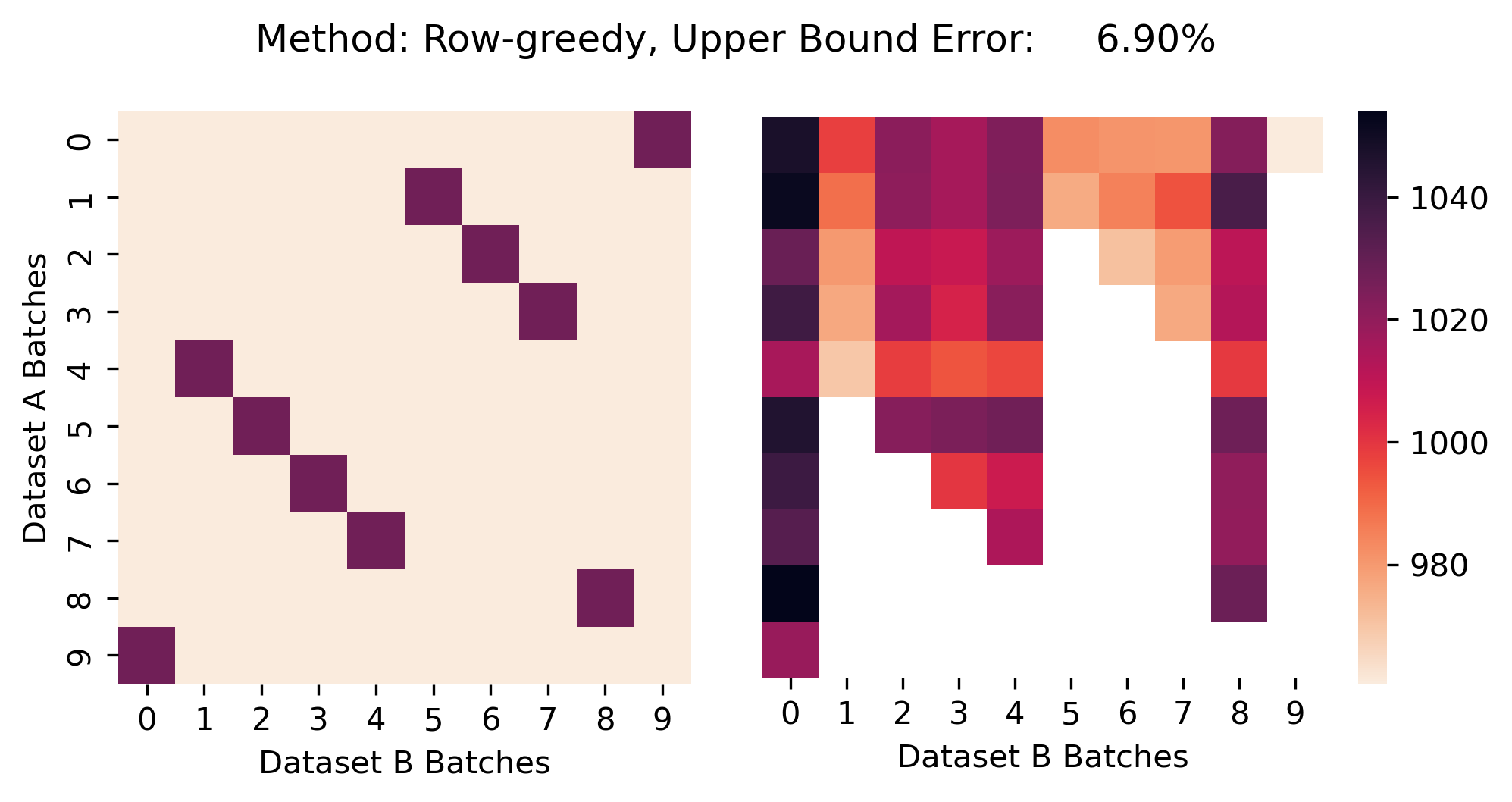}%
         \caption{Greedy batch matching}
         \label{fig:couplings_greedy}
     \end{subfigure}
     \hfill
     \begin{subfigure}[b]{0.32\textwidth}
         \centering
        \includegraphics[width=\textwidth, trim= {0 0.2cm 1.8cm 1cm}, clip]{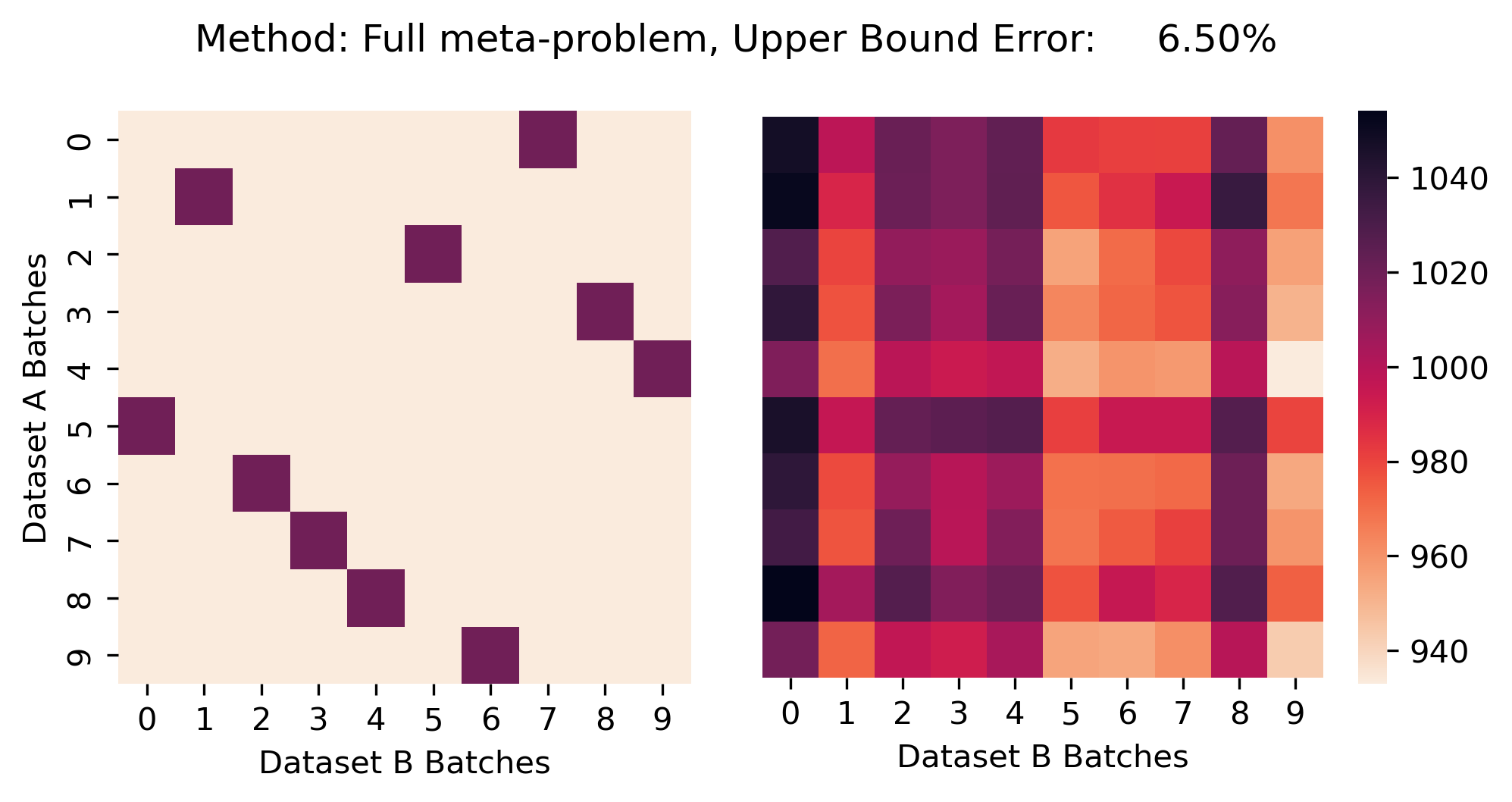}%
         \caption{\textsc{bhot}: optimal batch matching}
         \label{fig:couplings_bhot}
     \end{subfigure}
    \caption{\textbf{Batch-to-batch OT costs (R) and couplings (L)}. The usual averaging approach (\ref{fig:couplings_naive}) yields an upper bound to the full OT problem by solving $k$ mini-batch problems. The batch-hierarchical OT yields a tighter bound by solving $k^2$ OT problems and finding the optimal coupling between batches (\ref{fig:couplings_bhot}). The budget-constrained methods we propose (e.g., greedy matching \ref{fig:couplings_greedy}) interpolate between these two by solving a predetermined number of OT problems, creating a partial estimate of the batch-to-batch cost matrix (filled entries) and thus yielding a potentially sub-optimal batch coupling.}\label{fig:couplings}
\end{figure*}

The feasible solutions $\tilde{\P}$ characterized by Theorem~\ref{thm:blockwise_linear_primal} define a family of upper-bounds for \plaineqref{eq:primal_prob_full} parametrized by $\mathbf{W}$. The tightest bound in this family can be found by solving the following \textit{meta-problem} between mini-batches:
\begin{equation}\label{eq:primal_prob_meta}
	\min_{\mathbf{W}\in \U( \tilde{\a},  \tilde{\b})} \langle \mathbf{W}, \D \rangle = \min_{\mathbf{W}\in \U( \tilde{\a},  \tilde{\b})} \sum_{s,t}^K \omega_{st} d_{st}  	\tag{MP}
\end{equation} 
Note that \plaineqref{eq:primal_prob_meta} is itself an optimal transport problem where the mini-batches play the role of samples, their mass is proportional to the total mass of the samples in the batch, and the ground cost between mini-batches is the OT cost between their elements, i.e., $d_{st} = \OT(\alpha^s,\beta^t)$. We refer to the value of \plaineqref{eq:primal_prob_meta} (the tightest bound in its class) as the \textit{batch-hierarchical OT bound} (\textsc{bhot}). Thus, Theorem~\ref{thm:blockwise_linear_primal} in combination with meta-problem \plaineqref{eq:primal_prob_meta} shows an equivalence between constructing feasible solutions for---and therefore bounding---the full sample problem and finding a coupling between the batches of the two datasets. We will exploit this interpretation in the next section to propose computationally efficient bounds. 

Interestingly, the usual approach of estimating \plaineqref{eq:primal_prob_full} by averaging the solutions of $k$ problems on mini-batches sampled sequentially in parallel from $\alpha_N$ and $\beta_M$ \citep{genevay2018learning, salimans2018improving, fatras2021minibatch} is contained in the family of Theorem~\ref{thm:blockwise_linear_primal}. Indeed, taking $\omega_{st} = \tfrac{1}{k}$ if $s=t$ and $0$ otherwise, we obtain a block-diagonal solution $\tilde{\P}$ with optimal value $\frac{1}{k}\sum_{s=1}^k  \OT(\alpha^s,\beta^t)$. In particular, this implies that the \textsc{bhot} bound is, in the worst case, equal to the naive averaging bound. 

It must be emphasized that \textsc{bhot} is tightest only within the class of bounds of the form \plaineqref{eq:blockwise_weighted_coupling} but might not be the tightest one constructed from block-wise solutions. Indeed, there are other conceivable ways to combine such solutions into a feasible solution to \plaineqref{eq:primal_prob_full}, any of which provides a valid upper bound. The advantage of this particular family is that it allows for the computation of the upper bounds without having to explicitly construct any $N\times M$-sized object, instead relying on a linear combination of the values of mini-batch problems. Furthermore, it is optimizable by solving a $k$-by-$k$ problem only. We discuss computational aspects of this bound in detail in Section~\ref{sec:computation_ub}.

\begin{algorithm}[t]
\caption{Batch-Hierarchical OT Upper Bound}\label{algo:bhot_ub}
\KwIn{Data $\X \in \R^{N\times d}, \Y\in \R^{M\times d}$; num.~batches $k$;  }

\tcc{Stage 1: solving batch OT problems }
$n,m \gets \textup{GetBatchSize}(k,N,M)$\;
$\{ \X^1, \dots, \X^k \} \gets \text{Partition}(\X, n)$\;
$\{ \Y^1, \dots, \Y^k \} \gets \text{Partition}(\Y, m)$\; 
$\mathbf{D} \gets  \mathbf{0}_{k\times k}$ \;
\For{$(s,t) \in \llbracket k \rrbracket \times \llbracket k \rrbracket $}{
    $\mathbf{C}^{st} \gets \textup{PairwiseDistances}(\X^s, \Y^t)$\;
    $\mathbf{D}[s,t] \gets \OT(\mathbf{C}^{st}, \tfrac{1}{n}\ones_n, \tfrac{1}{m}\ones_m)$\;
}
\tcc{Stage 2: aggregating batch solutions}
$d \gets \OT(\mathbf{D}, \tfrac{1}{k}\ones_k, \tfrac{1}{k}\ones_k)$\;
\Return{d}
\end{algorithm}

Algorithm~\ref{algo:bhot_ub} summarizes the derivation above, providing pseudo-code to compute an upper bound to an OT problem by solving $k^2+1$ smaller problems.

\subsection{Complexity of mini-batch estimation}\label{sec:computation_ub}
Solving problem \plaineqref{eq:primal_prob_full} directly would amount to solving a single $N\times M = nk \times mk$ OT problem. This can be done exactly in $\tilde{O}((N+M)NM)$ (i.e., $\tilde{O}((n+m)nmk^3)$) time\footnote{We use $\tilde O(f)$ to denote $O(f\cdot\mathrm{polylog}(f))$.}  via the network simplex \citep{tarjan1997dynamic} or $\varepsilon$-approximately in $\tilde{O}(NM\varepsilon^{-3})=\tilde{O}(nmk^2\varepsilon^{-3})$ via the Sinkhorn algorithm \citep{altschuler2017nearlinear}. Either approach has $O(k^2nm)$ space complexity (since the entire $N \times M$ cost matrix must be computed and stored). This problem is not easily parallelizable. On the other hand, Algorithm~\ref{algo:bhot_ub} involves solving $k^2+1$ problems: $k^2$ problems \plaineqref{eq:primal_prob_batch} of size $n\times m$ and one final meta-problem \plaineqref{eq:primal_prob_meta} of size $k^2$, for a total $\tilde{O}((n+m)nmk^2)$ time complexity if solved exactly or $\tilde{O}(nmk^2\varepsilon^{-3})$ if solved approximately, and $O(nm + k^2)$ space complexity. \looseness=-1

Note that the time complexity of solving this batch-hierarchical problem approximately via the Sinkhorn algorithm is the same (up to constants) as that of solving the full problem \plaineqref{eq:primal_prob_full} when $k\leq n$ (which will typically be the case in practice). However, Algorithm~\ref{algo:bhot_ub} is parallelizable (over batch pairs), so its time complexity can be decreased to $\tilde{O}((n+m)nm)$ (or $\tilde{O}(nm \varepsilon^{-3})$) per processor if run in parallel on $k^2$ processors. \looseness=-1

In summary, if not parallelized \textsc{bhot} is advantageous memory-wise but is suboptimal time-wise, as it returns only an upper bound on the full problem at the same computational cost. If parallelized, it can achieve up to $O(k^2)$ speed-up over the full problem solution.  However, if extreme parallelism is not possible, or if the datasets are very large, even this upper bound might be prohibitive. Next, we investigate how to further reduce its computation. \looseness=-1

\begin{algorithm}[t]
\caption{Budget-Constrained OT Upper Bounds}\label{algo:bhot_ub_constrained}
\KwIn{Data $\X \in \R^{N\times d}, \Y\in \R^{M\times d}$;  num.~batches $k$; budget $B\in \{k,\dots, k^2\}$  }
\tcc{Stage 1: selecting batch pairs}
$n,m \gets \textup{GetBatchSize}(k,N,M)$\;
$\{ \X^1, \dots, \X^k \} \gets \text{Partition}(\X, n)$\;
$\{ \Y^1, \dots, \Y^k \} \gets \text{Partition}(\Y, m)$\; 
$ M \gets \textup{GetMatching}(B, \{\X^i\}, \{\Y^j\})$ \tcp*{$|M|=B$}
\tcc{Stage 2: solving batch problems }
$\mathbf{D} \gets \infty \cdot \ones_{k\times k}$ \;
\For{$(s,t) \in M$}{
    $\mathbf{C}^{st} \gets \textup{PairwiseDistances}(\X^s, \Y^t)$\;
    $\mathbf{D}[s,t] \gets \OT(\mathbf{C}^{st}, \tfrac{1}{n}\ones_n, \tfrac{1}{m}\ones_m)$\;
}
\tcc{Stage 3: aggregating batch solutions}
$d \gets \OT(\mathbf{D}, \tfrac{1}{k}\ones_k, \tfrac{1}{k}\ones_k)$\;
\Return{d}
\end{algorithm}

\subsection{Trading bound tightness for efficiency}\label{sec:tradeoffs}

In the previous section we showed how constructing a certain class of bounds based on solutions of mini-batch OT problems reduces to finding couplings between the mini-batches. %
A basic OT result states that for discrete distributions with uniform weights and equal sample size, the cost-minimizing optimal coupling $\P$ is in fact a permutation matrix (i.e., a matching) \citep[Prop.~2.1]{peyre2019computational}. In our setting, Assumption 1 guarantees this result for \plaineqref{eq:primal_prob_meta}, hence only $k$ of the optimal $\omega_{st}$ weights used by \textsc{bhot} are non-zero. Thus, despite requiring the solution of $k^2$ problems, this bound ultimately uses the value of only $k$ of them. This apparent waste suggests finding alternative methods to find matchings between mini-batches that require solving fewer mini-batch OT problems. \looseness=-1

For this purpose, we now conceptually decouple the process of constructing feasible solutions to \plaineqref{eq:primal_prob_full} from mini-batch solutions \plaineqref{eq:primal_prob_batch} into two steps: (i) matching and (ii) aggregation. The goal of the first step is to propose a low-cost matching between mini-batches, while the latter involves solving mini-batch OT problems between the selected matches and aggregating them into a bound for the full problem. In light of this new conceptual framework, the naive averaging bound can be understood as using a trivial (diagonal) matching between batches (with zero computational cost) and then solving $k$ OT problems between the matched batches in the aggregation step. At the opposite side of the spectrum, \textsc{bhot} requires solving all $k^2$ mini-batch problems to find the optimal matching by solving \plaineqref{eq:primal_prob_meta}, after which the values of the matched mini-batch problems are averaged. Below, we propose several methods that interpolate between these two extremes. For all of them, $B\in \N$ is a pre-specified budget specifying the maximum number of batch OT problems to solve.

\paragraph{Greedy matching.} At a high level, this method proceeds by matching one row (batch of $\X$) at a time to its most similar column (batch of $\Y$), after which this column is removed from the candidates. In its simplest form, this algorithm requires a $B=\binom{k}{2}$ budget. For lower budgets, we generalize this method by first splitting the $B $ among the rows using an allocation function $\mathfrak{a}:\llbracket k \rrbracket \rightarrow \N$ defined recursively as $\mathfrak{a}(1)=1, \mathfrak{a}(s+1) = \mathfrak{a}(s) + \textup{I}\bigl (B\geq \binom{k+1}{2} - \binom{k-s}{2} \bigr)$ for every  $s\in{1,\dots,k-1}$. Intuitively, this function allocates the budget in a diagonal fashion, starting from the first row, until it is exhausted. Once the budget is allocated, we proceed as before, but solving only as many OT problems per row as the allocation allows.

\paragraph{OT with missing costs.} We randomly sample $B$ pairs, without replacement, from the set of distinct index pairs in $\llbracket k \rrbracket\!\times\!\llbracket k \rrbracket$ that include each row and column index at least once. We solve the corresponding $B$ mini-batch problems and fill the corresponding entries of the cost matrix $\mathbf{W}$, assigning a value of $\infty$ (or, in practice, a finite scalar $M\gg 0$) 
to all other entries (corresponding to `missing' unsolved problems). We then proceed to solve the problem \plaineqref{eq:primal_prob_meta} as before. The infinity costs ensure that the optimal coupling has zero value for all entries corresponding to missing costs, and thus their value is not needed to compute the bound. We also consider a variant that greedily selects entries to fill (\textsc{bhot}-MissingGreedy, Appendix~\ref{sec:further_methods}). 

\paragraph{Tree-based batching.} Using probabilistic tree embedding techniques described in Section~\ref{sec:related}, we propose a 2-phase tree-based linear approximation algorithm for \textsc{bhot}. Informally, the first phase invokes the Quadtree algorithm to embed the batches into sparse vectors in $\ell_1$, and the second phase uses the Flowtree algorithm to find an approximately optimal matching between the batch embeddings. The resulting algorithm provably produces a multiplicative approximation for \textsc{bhot}:

\begin{theorem}[\textsc{bhot}-Tree Guarantee]\label{thm:tree}
Let $\Phi_{\mathbf{C}}$ be the aspect ratio\footnote{$\Phi_{\mathbf{C}}=\max_{i,j}\mathbf{C}_{ij}/\min_{i,j:\mathbf{C}_{ij}\neq0}\mathbf{C}_{ij}$.} of the full-sample cost matrix $\mathbf{C}_{ij}$. 
Suppose the costs are given by $\ell_1$-distances ($\mathbf C_{i,j}=c(X_i,Y_j)=\norm{X_i-Y_j}_1$).
\textsc{bhot}-Tree runs in time
$\tilde O(kd(n+m)\log\Phi_{\mathbf{C}})$, and computes a matching $\widetilde{\mathbf{W}}$ that with probability 
$0.99$ satisfies,
\[ \mathrm{BHOT} \leq \langle \widetilde{\mathbf{W}}, \D \rangle \leq O(\log^2(N) \cdot\log^2(d\Phi_{\mathbf{C}})\cdot\log k)\cdot\mathrm{BHOT} .\]
\end{theorem}

We note that the running time in the above theorem is only to compute the matching $\widetilde{\mathbf{W}}$ (the ``Matching Cost'' in Table~\ref{tab:my_label}). Computing the aggregate cost $\langle \widetilde{\mathbf{W}}, \D \rangle$ then requires solving the mini-batch OT problems corresponding to each edge in the matching (``Aggr. Cost'' in Table \ref{tab:my_label}). We also consider a simpler variant of this algorithm, that uses the mean of each batch as an embedding in $\ell_2$ instead of the Quadtree-based sparse $\ell_1$ embedding.

\ifbool{istwocolumn}{
\begin{table*}
    \centering
    \begin{tabular}{rcccc}
    \toprule
        Bound Method & Batch Metric & Matching & Matching Cost & Aggr. Cost \\
    \midrule 
    \textit{Budget-Agnostic}$\quad$\\
        Naive Avg. & $\OT(\alpha^s,\beta^t)$ & Random & $0$ & $k\cdot \textup{ot}(n,m)$ \\
        \textsc{bhot} & $\OT(\alpha^s,\beta^t)$ & OT & $k^2\cdot \textup{ot}(n,m) + \textup{ot}(k,k)$ &  $k^2 \cdot O(1)$\\
        \textsc{bhot}-Means & $d(\hat{\mu}_{\alpha},\hat{\mu}_{\beta})$ & OT & $k^2 \cdot O(d(n\!+\!m)) + \textup{ot}(k,k)$ & $k\cdot \textup{ot}(n,m)$\\
        \textsc{bhot}-Bures & $\BW(\alpha^s, \beta^t)$  & OT & $k^2 \cdot O(d^3 + d^2n)+ \textup{ot}(k,k)$ & $k\cdot \textup{ot}(n,m)$\\
        \textsc{bhot}-AvgDist & $\hat{\Exp}_{ij}[d(x_i,y_j)]$  & OT & $k^2 \cdot O(dnm)+ \textup{ot}(k,k)$ & $k\cdot \textup{ot}(n,m)$\\
    \textit{Budget-Constrained}\\
        \textsc{bhot}-Greedy & $\OT(\alpha^s,\beta^t)$ & Greedy & $\binom{k}{2}\cdot \textup{ot}(n,m)$ & $k\cdot O(1)$ \\    
        \textsc{bhot}-Missing & $\OT(\alpha^s,\beta^t)$ & OT & $B \cdot \textup{ot}(n,m) + \textup{ot}(k,k)$ & $k\cdot O(1)$ \\    
        \textsc{bhot}-MissingGreedy & $\OT(\alpha^s,\beta^t)$ & OT & $B \cdot \textup{ot}(n,m) + \textup{ot}(k,k)$ & $k\cdot O(1)$ \\            
        \textsc{bhot}-Tree & L1/Means & Flowtree &  $k\cdot \tilde O(d(n+m)\log\Phi_{\mathbf{C}})$ & $k\cdot \textup{ot}(n,m)$ \\
        \textsc{bhot}-Star & L1/Means & $\OT(\alpha^s,\beta^t)$ &  $k^{1+\rho}\cdot\tilde O(d(n+m)\log\Phi_{\mathbf{C}}+\textup{ot}(n,m))+\textup{ot}(k,k)$ & $k\cdot \textup{ot}(n,m)$ \\
    \bottomrule
    \end{tabular}
    \caption{Computational complexity of various upper bound methods. Here $n,m$ are the mini-batch sizes, $d$ their dimension, $k$ is the number of batches, $B$ is a pre-specified budget on the number of OT problems to solve, and $\textup{ot}(a,b)$ is the cost of solving an OT problem between $a$ and $b$ samples, which depending on the version of OT problem and solver can range between nearly-linear to cubic complexity on $nm$ (see Section~\ref{sec:computation_ub}).}
    \label{tab:my_label}
\end{table*}
}{
\begin{table}[t]
    \centering
    \resizebox{\textwidth}{!}{\begin{tabular}{rcccc}
    \toprule
        Bound Method & Batch Metric & Matching & Matching Cost & Aggr. Cost \\
    \midrule 
    \textit{Budget-Agnostic}$\quad$\\
        Naive Avg. & $\OT(\alpha^s,\beta^t)$ & Random & $0$ & $k\cdot \textup{ot}(n,m)$ \\
        \textsc{bhot} & $\OT(\alpha^s,\beta^t)$ & OT & $k^2\cdot \textup{ot}(n,m) + \textup{ot}(k,k)$ &  $k^2 \cdot O(1)$\\
        \textsc{bhot}-Means & $d(\hat{\mu}_{\alpha},\hat{\mu}_{\beta})$ & OT & $k^2 \cdot O(d(n\!+\!m)) + \textup{ot}(k,k)$ & $k\cdot \textup{ot}(n,m)$\\
        \textsc{bhot}-Bures & $\BW(\alpha^s, \beta^t)$  & OT & $k^2 \cdot O(d^3 + d^2n)+ \textup{ot}(k,k)$ & $k\cdot \textup{ot}(n,m)$\\
        \textsc{bhot}-AvgDist & $\hat{\Exp}_{ij}[d(x_i,y_j)]$  & OT & $k^2 \cdot O(dnm)+ \textup{ot}(k,k)$ & $k\cdot \textup{ot}(n,m)$\\
    \textit{Budget-Constrained}\\
        \textsc{bhot}-Greedy & $\OT(\alpha^s,\beta^t)$ & Greedy & $\binom{k}{2}\cdot \textup{ot}(n,m)$ & $k\cdot O(1)$ \\    
        \textsc{bhot}-Missing & $\OT(\alpha^s,\beta^t)$ & OT & $B \cdot \textup{ot}(n,m) + \textup{ot}(k,k)$ & $k\cdot O(1)$ \\    
        \textsc{bhot}-MissingGreedy & $\OT(\alpha^s,\beta^t)$ & OT & $B \cdot \textup{ot}(n,m) + \textup{ot}(k,k)$ & $k\cdot O(1)$ \\            
        \textsc{bhot}-Tree & L1/Means & Flowtree &  $k\cdot \tilde O(d(n+m)\log\Phi_{\mathbf{C}})$ & $k\cdot \textup{ot}(n,m)$ \\
        \textsc{bhot}-Star & L1/Means & $\OT(\alpha^s,\beta^t)$ &  $k^{1+\rho}\cdot\tilde O(d(n+m)\log\Phi_{\mathbf{C}}+\textup{ot}(n,m))+\textup{ot}(k,k)$ & $k\cdot \textup{ot}(n,m)$ \\
    \bottomrule
    \end{tabular}}
    \caption{Computational complexity of various upper bound methods. Here $n,m$ are the mini-batch sizes, $d$ their dimension, $k$ is the number of batches, $B$ is a pre-specified budget on the number of OT problems to solve, and $\textup{ot}(a,b)$ is the cost of solving an OT problem between $a$ and $b$ samples, which depending on the version of OT problem and solver can range between nearly-linear to cubic complexity on $nm$ (see Section~\ref{sec:computation_ub}).}
    \label{tab:my_label}
\end{table}
}

In our implementation, for simplicity, we replace the sparse $\ell_1$-embedding in the first phase, with a simpler variant that embeds each batch into $\R^d$ as the mean of the points in that batch. Furthermore, in the second phase, we use the tree constructed by Flowtree to compute additional edges on top of the matching it returns, thus extending its usability to budgets larger than $k$. 

\paragraph{Star-based matching.} The above \textsc{bhot}-Tree algorithm computes $k$ OT problems to obtain a polylogarithmic approximation for BHOT. 
Even though, as mentioned, we can (and will) use the tree to choose more than $k$ OT problems to fully compute, the approximation guarantee does not improve. 
We wish to extrapolate from this and obtain an algorithm that computes more than $k$ (but less than $k^2$) OT problems, and in return achieves a more accurate approximation. To this end, we replace the Flowtree-based second phase of BHOT-Tree with a sparse graph composed of a collection of stars \citep{har2013euclidean,carey2022stars}. The idea is to apply a hierarchy of gradually refined locality-preserving partitions (such as a Quadtree) and then pick a random point in each part of the partition and connect it with an edge to each other point in that part---thus forming a star-shaped subgraph. The edges of all stars thus added are the entries $\mathbf{C}^{s,t}$ of the batch-cost matrix that we compute to approximate BHOT. To adhere to the bipartite nature of the BHOT problem and avoid ``wasting'' OT computations between batches on the same side, we only star edges that connect a batch of $\alpha_n$ to a batch of $\beta_m$. The resulting algorithm has the following guarantee, offering a different efficiency to accuracy trade-off than BHOT-Tree.

\begin{theorem}[\textsc{bhot}-Star Guarantee]\label{thm:stars}
Let $\Phi_{\mathbf{C}}$ be the aspect ratio of the full-sample cost matrix $\mathbf{C}_{ij}$. 
Suppose the costs are given by $\ell_1$-distances ($\mathbf C_{i,j}=c(X_i,Y_j)=\norm{X_i-Y_j}_1$).
Let $\rho\in(0,1)$. 
\textsc{bhot}-Star runs in time $\tilde O((k^{1+\rho}+d)(n+m)\log\Phi_{\mathbf{C}}+k^{1+\rho}\cdot\textup{ot}(n,m)+\textup{ot}(k,k)$ and computes a matching $\widetilde{\mathbf{W}}$ that with probability 
$0.99$ satisfies,
\[ \mathrm{BHOT} \leq \langle \widetilde{\mathbf{W}}, \D \rangle \leq O(\rho^{-1}\log(N) \cdot\log(d\Phi_{\mathbf{C}}))\cdot\mathrm{BHOT} .\]
\end{theorem}

Algorithm~\ref{algo:bhot_ub_constrained} shows generic pseudocode for all the budget-constrained methods described so far, whereby the  $\texttt{GetMatching}$ function is method-specific.

\paragraph{\textsc{bhot} with approximate batch metrics.} In addition, we consider the following budget-agnostic bounds, which reduce computational complexity by approximating the batch-to-batch OT distance by cheaper proxy metrics. After solving the meta OT problem using these proxy costs, they solve exact OT problems only on the $k$ matched pairs obtained from the optimal coupling, and their values are averaged as a final step. For samples with empirical means $(\mu_\alpha, \mu_\beta)$ and covariances $(\Sigma_\alpha, \Sigma_\beta)$, we consider the following proxy metrics:
\begin{itemize}[topsep=0pt,parsep=3pt,partopsep=0pt,leftmargin=*]
    \item Distance between means: $\| \mu_{\alpha}- \mu_{\beta}\|_2$ (\textsc{bhot}-Means).
    \item Expected distance: $\Exp_{x,y} \| x - y\|_2$ (\textsc{bhot}-AvgDist).
    \item Bures-Wasserstein distance \citep{gelbrich1990formula, bhatia2019bures} (\textsc{bhot}-Bures):  
    \[\textup{BW}^2(\alpha, \beta)\!=\!\|\mu_\alpha - \mu_\beta\|_2^2 + \tr(\Sigma_{\alpha} + \Sigma_{\beta} -2 (\Sigma_{\alpha}^{\tfrac{1}{2}} \Sigma_\beta \Sigma_{\alpha}^{\tfrac{1}{2}})).\]
\end{itemize}
Here too, as in the Tree algorithm, our implementation uses the simpler mean-embedding of the batches instead of the sparse $\ell_1$-embedding.

Algorithm~\ref{algo:bhot_ub_proxy} (in Appendix~\cref{sec:further_methods}) provides pseudocode for these three methods.

\begin{figure*}
    \centering
    \includegraphics[width=\linewidth]{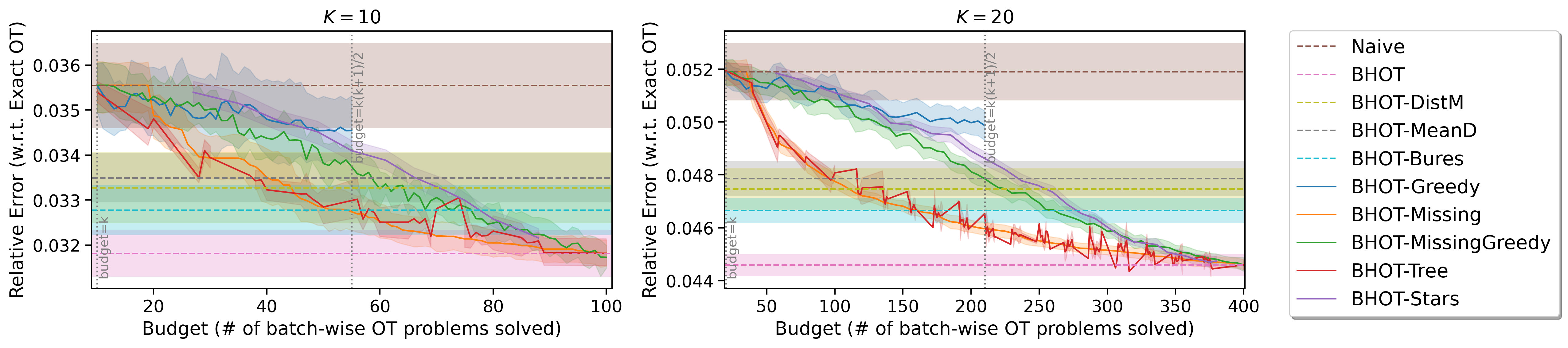}
    \caption{\textbf{Bound error vs.\ compute budget on MNIST vs.\ USPS}. For a given budget of batch-wise OT problems solved, we show the relative error in the bound obtained with each method, averaged across 10 repetitions. Unlike the budget-agnostic methods (dashed lines), the budget-constrained ones (solid) allow for fine-tuned bound tightness selection.}
    \label{fig:budget_vs_bound}
\end{figure*}

\section{Experiments}

\subsection{Data and setup}

For our first set of experiments, we compare various bounds on the OT distance between the \mnist and \usps datasets, the latter re-scaled to the $28\times28$ pixel size of the former to allow for direct comparison. To facilitate repeated computation of the various bounds across a spectrum of budgets, we take subsets of size $N=M=1000$ for each dataset. In all cases, we use the un-regularized $\OT$ distance as the batch-to-batch distance function, with the euclidean $\ell_2$ distance between images as ground metric. We use the Python Optimal Transport library \citep{flamary2021pot} to solve the OT sub-problems.

\subsection{Bound tightness vs.~computational budget}

We estimate the OT distance between the \mnist and \usps datasets using the various bounds proposed in this work (c.f.~Table~\ref{tab:my_label}). We take subsets of size $N=M=1000$ and vary the number of batches $k$, resulting in mini-batches of size $n=m=M/k$. For budget-constrained methods, we vary the budget between a lower limit of $k$ (what is used by the naive averaging baseline) and an upper limit of $k^2$ (what is used by the best-in-class \textsc{bhot} bound). We show the results for $k=\{10,20\}$ in Figure~\ref{fig:budget_vs_bound} and provide additional results in the Appendix. In these, we plot the relative error of the bounds with respect to the full-sample OT distance (i.e., the solution of \plaineqref{eq:primal_prob_full}), noting that budget-agnostic methods (dashed lines) appear as constant horizontal lines for comparison. \looseness=-1

The results confirm that our budget-constrained bounds smoothly interpolate between the usual naive averaging of batch solutions (dashed brown line in Fig.~\ref{fig:budget_vs_bound}) and the best-in-class \textsc{bhot} bound (pink, dashed). In this setting, the versions with missing costs and Flowtree approximation exhibit an overall superior cost-vs-tightness trade-off curve, particularly in the the smaller batch-size regime ($k=20$). Interestingly, the proxy-cost baselines ($\textsc{bhot-}[\textup{Bures}|\textup{MeanD}])$ are surprisingly tight, and in fact surpass the budget-constrained methods in the very-low-budget regime. This suggests that for this dataset, the first and second order moments of the per-batch distributions capture sufficient information to adequately approximate the OT distance between them. \looseness=-1

\subsection{Drift detection via two-sample tests}\label{sec:exp_twosample_testing}

Next, we investigate our methods in the context of distributional drift detection. We simulate a drift in \mnist by generating copies of this dataset where every image has been rotated $\theta$ degrees, with values of $\theta$ in $\{-4^\circ.,\dots,4^\circ\}$. Example images are provided in Appendix~\cref{sec:exp_details}. For a given rotated dataset, we investigate whether each method can distinguish it from the original version in a two-sample test. To obtain a statistical significance, we use permutation tests \citep{good2013permutation,ilmun2022minimax}, i.e., repeatedly mixing and shuffling the datasets and then comparing the distance between the original datasets and the random splits of these mixed datasets (details provided in Appendix~\ref{sec:exp_details}).

Figure~\ref{fig:alpha_power} shows the power of the two-sample test as a function of the rejection threshold $\alpha$ on the p-value, for two different degrees of rotation, using $k=20$ batches. \textsc{bhot}, \textsc{bhot}-Missing and \textsc{bhot}-Tree are significantly more sensitive to the drift, as shown by their strictly dominating power curves. For a fixed value of $\alpha=0.05$, Figure~\ref{fig:power_angle} shows the aggregated rejection rate of the tests for the entire range of rotations. While no method yields a test capable of distinguishing the samples up to rotations of $\pm2^{\circ}$, most quickly acquire specificity for larger distortions, although the naive and \textsc{bhot}-DistM methods remain at low rejection rates. \looseness=-1

\ifbool{istwocolumn}{
\begin{figure}
    \centering
    \includegraphics[width=\linewidth]{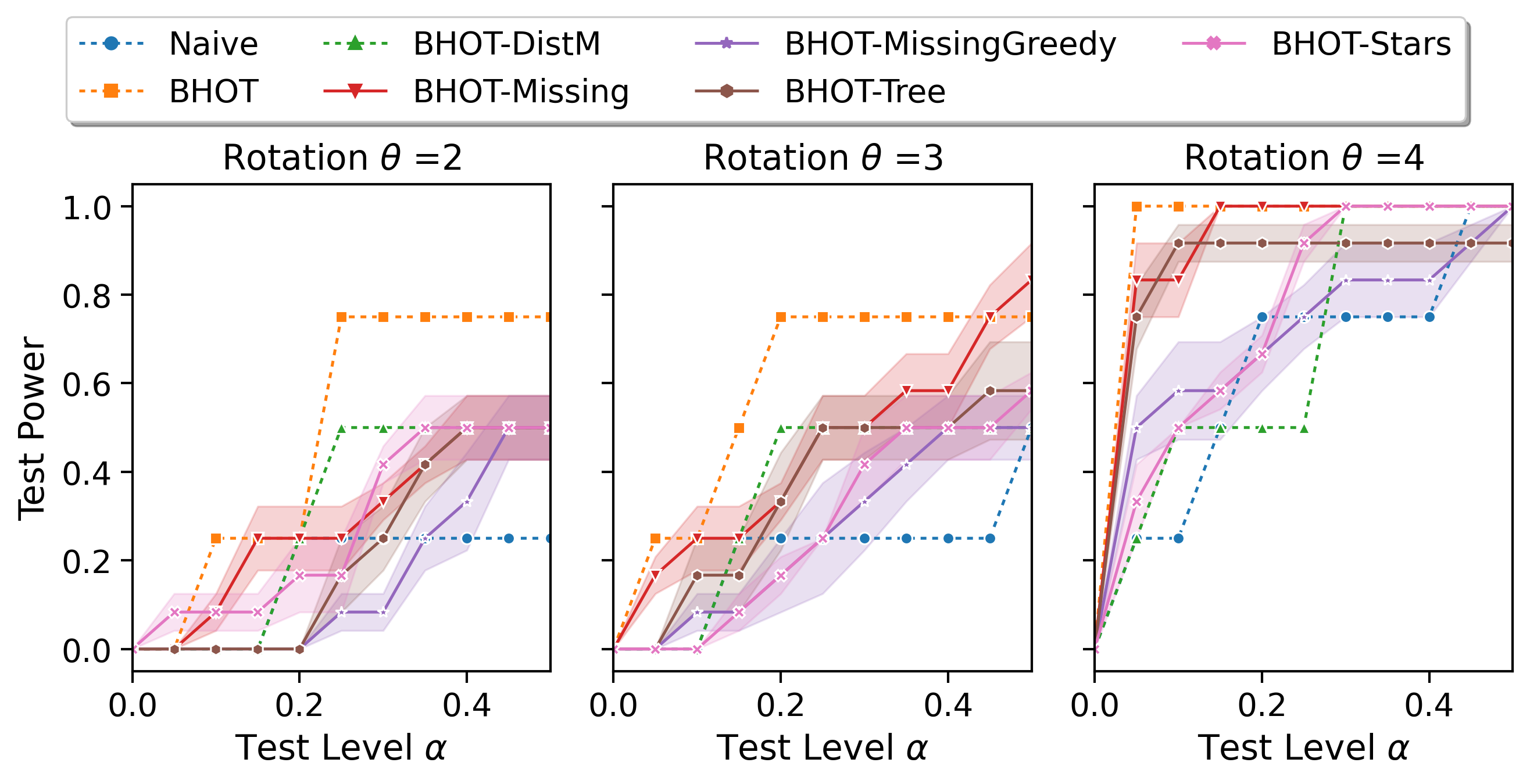}
    \caption{\textbf{Test power on MNIST vs.~Rotations}. 
    As the degree of rotation increases, the power of the two-sample tests based on all the bounds improves, although some (including naive averaging) do so more slowly. 
    }
    \label{fig:alpha_power}
\end{figure}
\begin{figure}
    \centering
    \includegraphics[width=\linewidth]{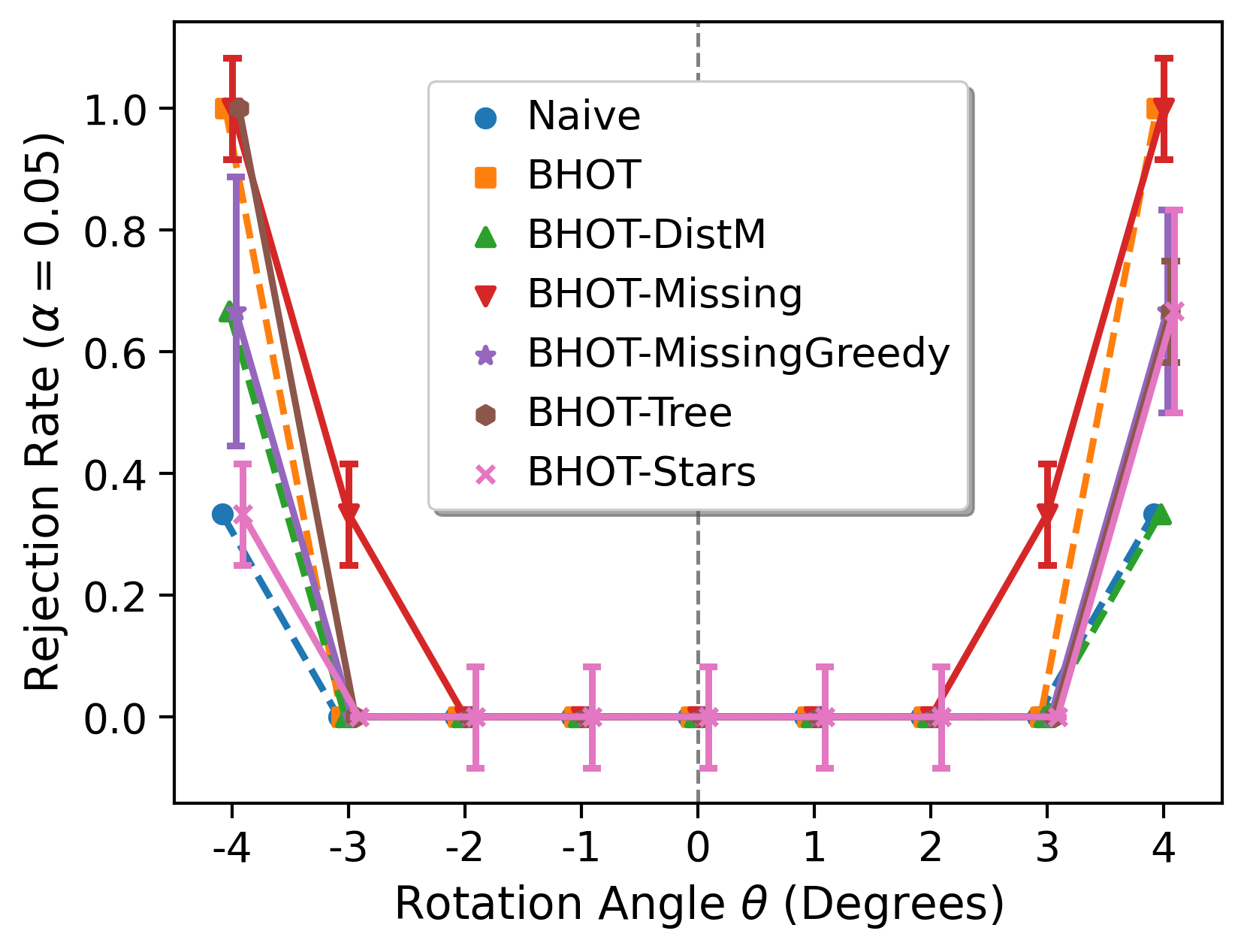}
    \caption{\textbf{Test rejection rate on MNIST vs rotations}. The two-sample test rejection rate for various bounding methods as a function of the degree of rotation. Most methods, although not naive averaging, can reliably distinguish the original and rotated datasets at $\theta\geq 4^\circ$ rotation angle.}
    \vspace{-0.2cm}
    \label{fig:power_angle}
\end{figure}
}{
\begin{figure}
\centering
\begin{minipage}{.62\textwidth}
    \vspace{-0.2cm}
  \centering
    \includegraphics[width=\linewidth]{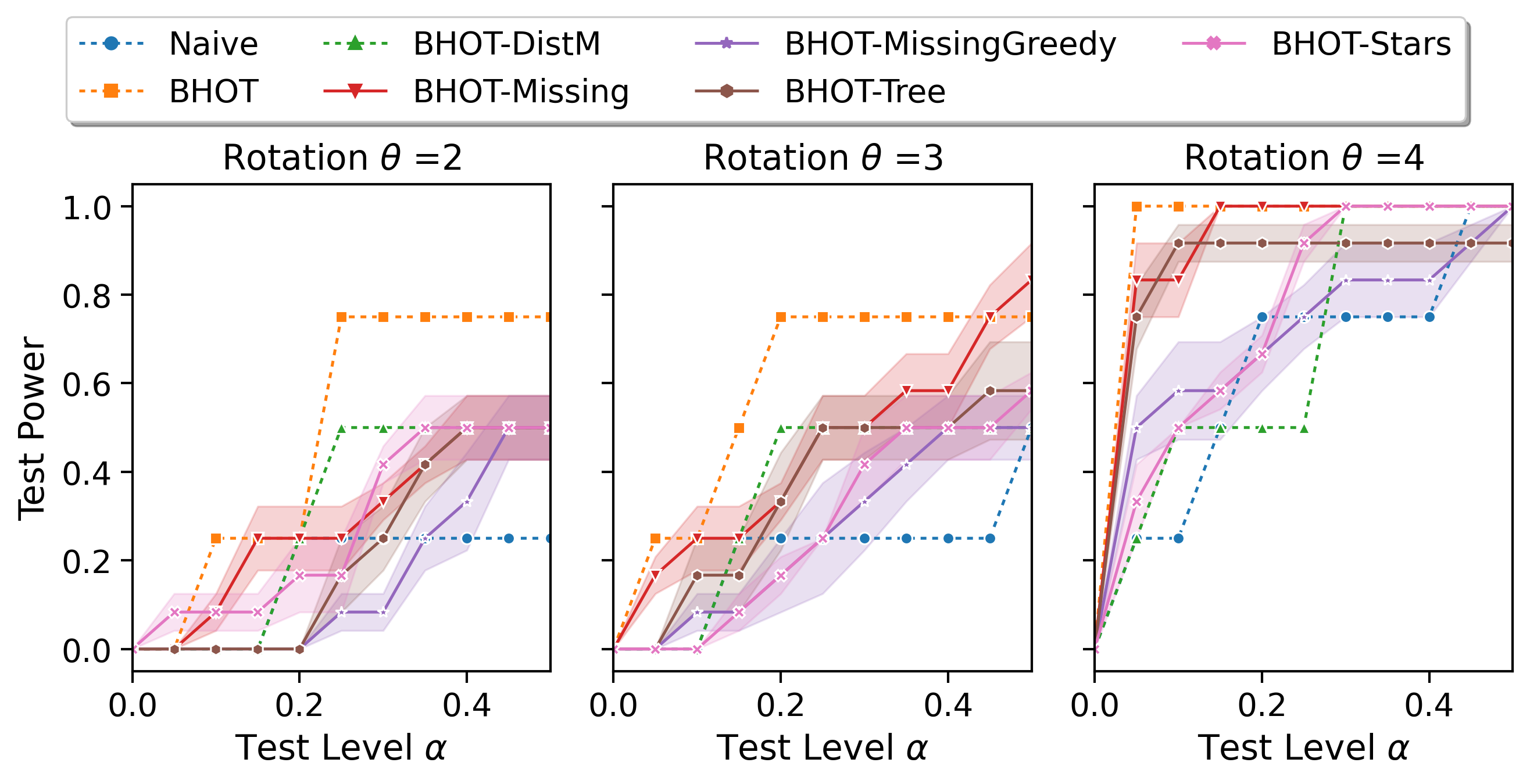}
  \captionof{figure}{\textbf{Test power on MNIST vs.~rotations}. 
    As the degree of rotation increases, the power of the two-sample tests based on all the bounds improves, although some (including naive averaging) do so more slowly. Some of the budget-constrained bounds yield tests roughly as powerful as that of the best-in-class \textsc{bhot} that require solving only a subset of the mini-batch OT problems. 
    }
  \label{fig:alpha_power}
\end{minipage}%
\hfill
\begin{minipage}{.35\textwidth}
  \centering
    \includegraphics[width=\linewidth]{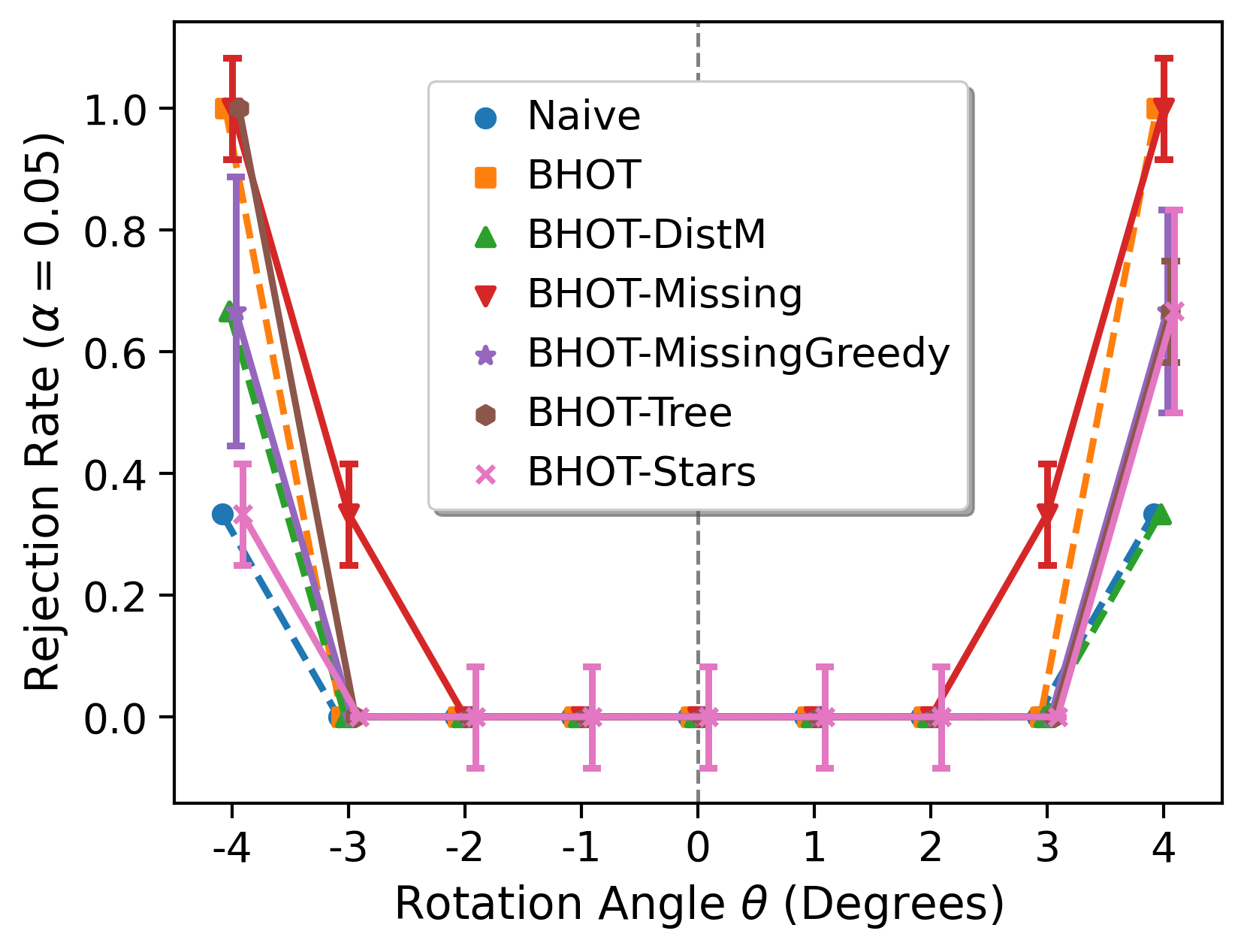}
  \captionof{figure}{\textbf{Test rejection rate on MNIST vs.\ rotations}. The two-sample test rejection rate for various bounding methods as a function of the degree of rotation. Most methods, although not naive averaging, can reliably distinguish the original and rotated datasets at $\theta\geq 4^\circ$ rotation angle.}
\label{fig:power_angle}
\end{minipage}
\end{figure}

}

\section{Discussion}

We have presented a family of bounds for the optimal transport problem that require solving only smaller problems between mini-batches. These bounds allow for trading off bound tightness for computational efficiency, and some of them come with provable guarantees. Although we have focused on upper bounds, we have also shown that a similar approach can be used to obtain analogous lower bounds. Our results suggest that the methods based on flowtree approximation and OT with missing costs tend to yield tighter bounds for every budget regime, but it is an interesting question for future work whether this trend is preserved for different datasets. Furthermore, the family of bounds considered here is certainly not the only one that can be constructed from solutions to sub-problems. It is left as an open question whether other families of bounds could provide better---perhaps even Pareto-optimal---budget-tightness trade-offs. 

\printbibliography

\opt{arxiv}{
    \pagebreak
}

\opt{aistats}{
    \pagebreak
    \clearpage
    \onecolumn
}

\appendix

\section{Lower Bounds via Dual Solutions}\label{sec:lower_bounds}

The dual of the full-sample OT problem \plaineqref{eq:primal_prob_full} between empirical measures $\alpha_N$ and $\beta_M$ is 
\begin{equation}\label{eq:dual_problem_full}
   \sup_{(\f, \g) \in \mathcal{R}(\mathbf{C})}  \langle \f, \a \rangle +  \langle \g, \b \rangle \tag{FD}
\end{equation}
where the supremum is taken over the set of feasible potentials: $\mathcal{R}(\mathbf{C})= \{ (\f,\g) \in \R^{N\times M} \st \f \oplus \g \leq \mathbf{C} \}$. Using the block structure induced by the partition into $k$ mini-batches of each sample, we can equivalently write \plaineqref{eq:dual_problem_full} as
\begin{equation}\label{eq:dual_problem_full_batch}
    \sup_{\{\f^s\}, \{\g^t\} } \sum_{s=1}^k  \langle \f^s, \a^s \rangle +  \sum_{t=1}^k \langle \g^t, \b^t \rangle \qquad \text{subject to} \qquad \f^s \oplus \g^t \leq \mathbf{C}^{st} \quad \forall s,t \in \llbracket k \rrbracket.  \tag{FD'}
\end{equation}
In addition, the partitioning also defines $k^2$ individual mini-batch problems (the duals of the \plaineqref{eq:primal_prob_batch} problems): 
\begin{equation}\label{eq:dual_problem_batch}
   \sup_{(\f^{st}, \g^{st}) \in \mathcal{R}(\mathbf{C}^{st})}  \langle \f^{st}, \a^s \rangle +  \langle \g^{st}, \b^t \rangle, %
   \tag{BD}
\end{equation}
which, analogously as for the primal, differ from \plaineqref{eq:dual_problem_full_batch} in that they are a collection of decoupled problems, while the latter is a single coupled problem. However, an important difference with the primal formulation is that here problem \plaineqref{eq:dual_problem_full_batch} involves a total of $2k$ potentials, while the problems \plaineqref{eq:dual_problem_batch} involve a total of $2k^2$ potentials. This seemingly subtle discrepancy will become prove crucial below.

Let $(\f^{st}_*, \g^{st}_*)$ denote the optimal pair for each problem \plaineqref{eq:dual_problem_batch}. Our goal is to construct a feasible solution to \plaineqref{eq:dual_problem_full_batch} using these mini-batch-wise solutions. Because of the discrepancy described above, for a given `block' of the full problem potential $\f^s$ we have multiple candidates (namely any $\f^{st}_*$ for a fixed $s$), and analogously for $\g^t$. Again, this is in contrast to the primal case studied before, where there is a unique correspondence between blocks of the full-sample solution and the mini-batch optimal solutions. The next theorem proposes one possible way to resolve this multiplicity and construct solutions to the full-sample problem using the mini-batch ones. 

\begin{theorem}\label{thm:additive_diagonal_dual}
    Under Assumption 1, let $\{\f^{st}_*, \g^{st}_*\}_{s,t=1}^{k,k}$ be pairs of optimal solutions to the batch-wise dual problems \plaineqref{eq:dual_problem_batch}. We define\footnote{Equivalently, $\tilde{\f} = \text{vec}(\F + \ones_n \otimes \u^T)$, where the columns of $\F$ are $\f^{ss}_*$, and analogously for $\tilde{\g}$.} $\tilde{\f}\in \R^N,\tilde{\g} \in \R^M$ block-wise as:
    \vspace{-0.1cm}
    \begin{align*}
       \forall s \in \llbracket k \rrbracket: \medspace &[\tilde{\f}_i]_{i \in \mathrm{B}_x^s} = \f^{ss}_* + u_s \ones_n, \\
       \forall t \in \llbracket k \rrbracket: \medspace &[\tilde{\g}_j]_{j \in \mathrm{B}_y^t} = \g^{tt}_* + v_t \ones_m   
    \end{align*}
    for some values  $\u = (u_1, \dots, u_k)$ and $\v =(v_1, \dots, v_k)$. If $\u$ and $\v$ satisfy $ \u \oplus \v \leq \mathbf{K}$, where $\mathbf{K}$ is a $k$-by-$k$ matrix with entries $K_{s,t} = \min_{j}[\g^{st}_* - \g^{tt}_*]_{j} -  \max_{i}[\f^{ss}_* - \f^{st}_*]_{i}$, then $(\tilde{\f}, \tilde{\g})$ is a feasible pair for \plaineqref{eq:dual_problem_full} with objective value:
    \begin{equation}\label{eq:dual_reduction_obj}
        \langle \tilde{\f}, \a \rangle+\langle \tilde{\g}, \b \rangle =  \langle \u , \tilde{\a} \rangle+\langle \v, \tilde{\b} \rangle +\sum_{s=1}^k	\tfrac{1}{k} \OT(\alpha^s,\beta^s), 
    \end{equation}
    where $\tilde{\a},\tilde{\b}$ are the vectors of aggregated mini-batch masses defined in Section~\ref{sec:partitioning}.
\end{theorem}

The feasible solutions $(\tilde{\f},\tilde{\g})$ characterized by Theorem~\ref{thm:additive_diagonal_dual} define a family of lower bounds for \plaineqref{eq:dual_problem_full} parametrized by $\u$ and $\v$. The tightest bound in this family can be found by solving a \textit{meta-problem} between mini-batches:
\begin{equation}\label{eq:dual_problem_meta}
	   \sup_{(\u, \v) \in \R^k \times \R^k } \langle \u , \p \rangle + \langle \v, \q \rangle \quad \text{s.t.} \quad \u \oplus \v \leq \mathbf{K}  \tag{MD}
\end{equation}
where $\K$ is defined in Theorem~\ref{thm:additive_diagonal_dual}. This meta-problem is (the dual of) a $k$-by-$k$ Optimal Transport problem, analogous to \plaineqref{eq:primal_prob_meta}, whose solution immediately yields the desired bound (via~\plaineqref{eq:dual_reduction_obj}) without having to explicit construct the full $\tilde{\f}$ or $\tilde{\g}$. It is also worth noting that although Theorem~\ref{thm:additive_diagonal_dual} defines these candidate dual solutions using the optimal solutions to the problems along the diagonal $\f^{ss}_*$ and $\g^{tt}_*$,  any other $\f^{(s,s')}_*, s\neq s'$ and $\g^{(t',t)}_*, t \neq t'$ could have been used as template. This would naturally change the form of the last term in~\plaineqref{eq:dual_reduction_obj}. We leave the optimization of this choice for future work.

\pagebreak
\clearpage

\section{Proofs}

\ifbool{istwocolumn}{
\begin{table}
    \centering
    \begin{tabular}{@{}cp{12cm}@{}}
        \toprule
        Symbol &  Description \\
        \midrule
        \textit{Full Problem} \plaineqref{eq:primal_prob_full}: \hspace{1.4cm} \\
        $\alpha_N, \beta_M$ & Empirical distributions $\alpha_N = \sum_{i=1}^N a_i \delta_{\rvx^i} , \beta_M = \sum_{j=1}^M b_j \delta_{\rvy^j}$ \\ 
        $\a \in \Delta_N, \b \in \Delta_M $ & Probability histograms of empirical distributions $\alpha_N$ and $\beta_M$ \\
        $\X \in \R^{N\times d}, \Y\in \R^{M\times d}$ & Dataset matrices of empirical distributions $\alpha_N$ and $\beta_M$, i.e., $\X_{i,:} = \rvx^i, \Y_{j,:} = \rvy^j$ \\[.5em]
        \textit{Mini-Batch Problems}  (\ref{eq:primal_prob_full_blockified},~\ref{eq:primal_prob_batch}): \\
        $\mathrm{B}^s_x \in \wp(\llbracket N \rrbracket), \mathrm{B}^t_y \in \wp(\llbracket M \rrbracket)$ & Subsets of indices defining mini-batch $s$ of $\alpha_N$ and mini-batch $t$ of $\beta_M$ \\
        $\X^s \in \R^{n\times d}, \Y^t\in \R^{m\times d}$ & Mini-batch matrices: $\X^s = [\X_{i,:}]_{i \in  \mathrm{B}^s_x}$, $\Y^t = [\Y_{j,:}]_{j \in  \mathrm{B}^t_y}$ \\
        $\hat{\a}^s \in \R^n_+, \hat{\b}^t \in \R^M_+$ & Per-batch un-normalized histograms, i.e., $\hat{\a}^s = [a_i]_{i \in \mathrm{B}^s_x}, \hat{\b}^t = [b_j]_{j \in \mathrm{B}^t_y}$\\
        $\a^s \in \Delta_n, \b^t \in \Delta_m$ & Per-batch probability (i.e., normalized) histograms $\a^s = \frac{\hat{\a}^s}{\sum \hat{\a}^s}, \b^t = \frac{\hat{\b}^t}{\sum \hat{\b}^t}$\\
        $\tilde{\alpha}^s, \tilde{\beta}_t$ & Unnormalized per-batch distributions $\tilde{\alpha}^s = \sum_{i \in \mathrm{B}^s_x}  a_i \delta_{\rvx^i} , \tilde{\beta}^t = \sum_{j \in \mathrm{B}^t_y}  b_j \delta_{\rvy^j} $ \\ 
        $\alpha^s, \beta_t$ & Normalized per-batch distributions $\alpha^s = \sum_{i =1}^n a^s_{i} \delta_{\rvx^{\sigma(i)}} , \beta^t = \sum_{j=1}^m b_j \delta_{\rvy^{\sigma(j)}} $ \\[.5em]
        \textit{Meta Problem} \plaineqref{eq:primal_prob_meta}: \hspace{1.3cm} \\
        $\tilde{\a}^s \in \Delta_k, \tilde{\b}^t \in \Delta_k$ & Probability histograms over mini-batches (uniform under Assumption 1)  \\
        $\D \in \R^{k\times k}$ & Matrix of transport costs between mini-batches, $D_{st} = \OT(\alpha^s, \beta^t)$ \\
        $\W \in \R^{k\times k }_+ $ & Transport plan between mini-batches  \\
        \bottomrule
    \end{tabular}
    \caption{Glossary of notation and objects used in this manuscript. }
    \label{tab:glossary}
\end{table}
}{
\begin{table}
    \centering
    \resizebox{\textwidth}{!}{\begin{tabular}{@{}cp{12cm}@{}}
        \toprule
        Symbol &  Description \\
        \midrule
        \textit{Full Problem} \plaineqref{eq:primal_prob_full}: \hspace{1.4cm} \\
        $\alpha_N, \beta_M$ & Empirical distributions $\alpha_N = \sum_{i=1}^N a_i \delta_{\rvx^i} , \beta_M = \sum_{j=1}^M b_j \delta_{\rvy^j}$ \\ 
        $\a \in \Delta_N, \b \in \Delta_M $ & Probability histograms of empirical distributions $\alpha_N$ and $\beta_M$ \\
        $\X \in \R^{N\times d}, \Y\in \R^{M\times d}$ & Dataset matrices of empirical distributions $\alpha_N$ and $\beta_M$, i.e., $\X_{i,:} = \rvx^i, \Y_{j,:} = \rvy^j$ \\[.5em]
        \textit{Mini-Batch Problems}  (\ref{eq:primal_prob_full_blockified},~\ref{eq:primal_prob_batch}): \\
        $\mathrm{B}^s_x \in \wp(\llbracket N \rrbracket), \mathrm{B}^t_y \in \wp(\llbracket M \rrbracket)$ & Subsets of indices defining mini-batch $s$ of $\alpha_N$ and mini-batch $t$ of $\beta_M$ \\
        $\X^s \in \R^{n\times d}, \Y^t\in \R^{m\times d}$ & Mini-batch matrices: $\X^s = [\X_{i,:}]_{i \in  \mathrm{B}^s_x}$, $\Y^t = [\Y_{j,:}]_{j \in  \mathrm{B}^t_y}$ \\
        $\hat{\a}^s \in \R^n_+, \hat{\b}^t \in \R^M_+$ & Per-batch un-normalized histograms, i.e., $\hat{\a}^s = [a_i]_{i \in \mathrm{B}^s_x}, \hat{\b}^t = [b_j]_{j \in \mathrm{B}^t_y}$\\
        $\a^s \in \Delta_n, \b^t \in \Delta_m$ & Per-batch probability (i.e., normalized) histograms $\a^s = \frac{\hat{\a}^s}{\sum \hat{\a}^s}, \b^t = \frac{\hat{\b}^t}{\sum \hat{\b}^t}$\\
        $\tilde{\alpha}^s, \tilde{\beta}_t$ & Unnormalized per-batch distributions $\tilde{\alpha}^s = \sum_{i \in \mathrm{B}^s_x}  a_i \delta_{\rvx^i} , \tilde{\beta}^t = \sum_{j \in \mathrm{B}^t_y}  b_j \delta_{\rvy^j} $ \\ 
        $\alpha^s, \beta_t$ & Normalized per-batch distributions $\alpha^s = \sum_{i =1}^n a^s_{i} \delta_{\rvx^{\sigma(i)}} , \beta^t = \sum_{j=1}^m b_j \delta_{\rvy^{\sigma(j)}} $ \\[.5em]
        \textit{Meta Problem} \plaineqref{eq:primal_prob_meta}: \hspace{1.3cm} \\
        $\tilde{\a}^s \in \Delta_k, \tilde{\b}^t \in \Delta_k$ & Probability histograms over mini-batches (uniform under Assumption 1)  \\
        $\D \in \R^{k\times k}$ & Matrix of transport costs between mini-batches, $D_{st} = \OT(\alpha^s, \beta^t)$ \\
        $\W \in \R^{k\times k }_+ $ & Transport plan between mini-batches  \\
        \bottomrule
    \end{tabular}}
    \caption{Glossary of notation and objects used in this manuscript. }
    \label{tab:glossary}
\end{table}
}

\subsection{Proof of Theorem~\ref{thm:blockwise_linear_primal}}\label{sec:proof_upperbound}

Let $\tilde{\P}$ be defined as in the statement of Theorem~\ref{thm:blockwise_linear_primal}, i.e., $[\tilde{\P}_{ij}]_{i \in \mathrm{B}^s_x, j\in \mathrm{B}^t_y} \eqdef \omega_{st} \P^{st}_*$. Plugging this into the 
constraints in \plaineqref{eq:primal_prob_full_blockified}, we obtain:
\begin{equation}
	\sum_{t=1}^k \omega_{st} \P^{st}_* \ones_m = \hat{\a}^{s}  \quad \forall s \in \llbracket k \rrbracket \quad\text{and}\quad
	\sum_{s=1}^k \omega_{st} (\P^{st}_*)^\top \ones_n = \hat{\b}^{t} \quad \forall t \in \llbracket k \rrbracket.
\end{equation}
where $\hat{\a}, \hat{\b}$ are the unnormalized histograms corresponding to mini-batches $s$ and $t$ from datasets $\mathbf{X}$ and $\mathbf{Y}$ respectively. Since all the $\P^{st}_*$ are optimal for their corresponding problems \plaineqref{eq:primal_prob_batch}, in particular they satisfy $\P^{st}\in \mathbf{U}(\a^s, \b^t)$, so the above system simplifies as: 
\begin{equation}
	\sum_{t=1}^k \omega_{st} \a^{s} = \hat{\a}^{s}  \quad \forall s \in \llbracket k \rrbracket \quad\text{and}\quad
	\sum_{s=1}^t \omega_{st} \b^{t} = \hat{\b}^{t} \quad \forall t \in \llbracket k \rrbracket,
\end{equation}
or equivalently, 
\begin{equation}\label{eq:meta_constraints}
	\sum_{t=1}^k \omega_{st} = \ones_n^\top \hat{\a}^{s}\quad \forall s \in \llbracket k \rrbracket \quad\text{and}\quad
	\sum_{s=1}^k \omega_{st} = \ones_m^\top  \hat{\b}^{t} \quad \forall t \in \llbracket k  \rrbracket.
\end{equation}
where we have used the fact that $\a^s, \b^t$ are normalized histograms (i.e., $\sum_i \a^s_i =1 =\sum_j\b^t_j$).

Let $\mathbf{W}$ be the $k$-by-$k$ matrix with entries $\omega_{st}$. The system of equalities above can be written compactly in matrix-vector form as $\mathbf{W} \ones_K = \tilde{\a}, \mathbf{W}^\top\ones_K = \tilde{\b}$, using $\tilde{\a}, \tilde{\b}$ the vectors of per-batch total mass as defined in Section~\ref{sec:partitioning}.

Thus, $\tilde{\P}$ satisfies the constraints of problem \plaineqref{eq:primal_prob_full} if and only if the matrix $\mathbf{W}$ satisfies the row- and column-sum constraints in \plaineqref{eq:meta_constraints}, as claimed. This coupling has objective value: 
\[ \langle \tilde{\P}, \mathbf{C} \rangle =\sum_{s,t=1}^k \langle \P^{st}, \mathbf{C}^{st} \rangle =  \sum_{s,t=1}^k  \omega_{st} \langle \P^{st}_*, \mathbf{C}^{st} \rangle  = \sum_{s,t}^k \omega_{st} \OT(\alpha^{s},\beta^{t})=\langle \mathbf{W}, \D \rangle \]
where we have used the optimality of $\P^{st}_*$ for its respective mini-batch problem. This completes the proof. 

\qed

\subsection{Proof of Theorem~\ref{thm:additive_diagonal_dual}}\label{sec:proof_lowerbound}
We will use the following lemma: 
\begin{lemma}\label{lemma:oplus_ineq}
    Let $\a , \x \in \R^n$ and $\b,\y \in \R^m$, and let $\oplus$ be defined as above. Then $\a \oplus \b \leq \x \oplus \y$ if and only if $\max_{i} a_i  - x_i \leq \min_{j} y_j - b_j$.
\end{lemma}

Let $\tilde{\f},\tilde{\g}$ be defined as in the statement of Theorem~\ref{thm:additive_diagonal_dual}. Here we will drop $*$ from $\f^{ss}_*$ and $\g^{tt}_*$ for notational simplicity, but it should not be forgotten that these are optimal dual solutions to their respective batch problems. Now, suppose $\u \oplus \v \leq \K$. We want to show that $\tilde{\f} \oplus \tilde{\g} \leq \mathbf{C}$, or equivalently (using the block structure of these two matrices) that $\tilde{\f}^s \oplus \tilde{\g}^t \leq \mathbf{C}^{s,t}$ for every $s,t \in \llbracket k \rrbracket$. For a given $(s,t)$ block, the following inequalities are equivalent:
\begin{align*}
    u_s + v_t & \leq K_{s,t}  \eqdef \min_{j}[\g^{st} - \g^{tt}]_{j} -  \max_{i}[\f^{ss} - \f^{st}]_{i}  \\
    \max_{i}[\f^{ss} - \f^{st}]_{i}  +  u_s &\leq   \min_{j}[\g^{st} - \g^{tt}]_{j} - v_t \\
    \max_{i}[\f^{ss} + u_s\ones_n - \f^{st}]_{i} &\leq  \min_{j}[ \g^{st} - \g^{tt} - v_t\ones_m ]_{j} \\
    \max_{i}[\tilde{\f}^s - \f^{st}]_{i} &\leq  \min_{j}[ \g^{st} - \tilde{\g}^t ]_{j}                
\end{align*}
and, by Lemma~\ref{lemma:oplus_ineq},
\begin{equation*}
    \tilde{\f}^s \oplus  \tilde{\g}^t \leq  \f^{st} \oplus \g^{st} \leq \mathbf{C}^{st}
\end{equation*}
where the last inequality holds because $\f^{st}$ and $\g^{st}$ are an optimal (and therefore feasible) pair for the batch-wise problem \plaineqref{eq:dual_problem_batch}. Thus, $\u \oplus \v \leq \K$ implies $\tilde{\f}  \oplus \tilde{\g} \leq \mathbf{C}$, as claimed. The objective value of this pair is
\begin{align*}
    \langle \tilde{\f}, \a \rangle + \langle \tilde{\g}, \b \rangle &= \sum_{s=1}^k \langle \f^{ss} + u_s \ones_n ,\a^s \rangle + \sum_{t=1}^k \langle \g^{tt} + v_t \ones_n, \b^t \rangle \\ 
    &= \sum_{s=1}^k \langle \f^{ss} ,\a^s \rangle  + u_s \langle \ones_n ,\a^s \rangle + \sum_{t=1}^k \langle \g^{tt}, \b^t \rangle + v_t \langle  \ones_n, \b^t \rangle \\
    &= \langle \u , \tfrac{1}{k} \ones_k \rangle + \langle \v , \tfrac{1}{k} \ones_k \rangle + \sum_{s=1}^k \tfrac{n}{N}\langle  \f^{ss} ,\tilde{\a}^s \rangle  + \tfrac{m}{M} \langle \g^{ss}, \tilde{\b}^s \rangle \\
    &= \langle \u , \p   \rangle + \langle \v , \q  \rangle + \sum_{s=1}^k \tfrac{1}{k} \OT(\alpha^s,\beta^s)
\end{align*}
where the last equality follows from the optimality of $\f^{st}$ and $\g^{st}$.
\qed

\subsection{Proof of Theorem~\ref{thm:tree}}
For this theorem we assume that the ground metric for the points is embedded into $\ell_1$, i.e., that $\mathcal X \cup \mathcal Y\subset \R^d$ and that the full-sample cost matrix $\mathbf C$ is given by $\mathbf C_{i,j}=c(X_i,Y_j)=\norm{X_i-Y_j}_1$. Note that this is more general than assuming Euclidean costs, since Euclidean metrics embeds isometrically into $\ell_1$, and furthermore we can efficiently embed the Euclidean costs between points in $\mathcal X \cup \mathcal Y$ into $\ell_1$ with constant distortion (which only changes the hidden constants in the theorem statement), by a random rotation of the dataset. We furthermore assume for simplicity that $\mathcal X$ and $\mathcal Y$ are disjoint, or equivalently, that the full-sample cost matrix $\mathbf C$ contains no zero entries. Allowing $\mathcal X$ and $\mathcal Y$ to overlap is a simple extension of the proof below, which does not qualitatively change the result but would burden the proof with technical details.

The proof has two steps. First, we embed the batches into sparse vectors in $\ell_1$ with polylogarithmic distortion, using the techniques of \cite{andoni2008earth}. Once the batches are embedded in $\ell_1$, we can use the Flowtree technique from \cite{backurs2020scalable} to compute an approximately optimal matching between the batches in nearly linear time. The overall approximation factor is the product of the approximation factors of each of the two steps, each of which is polylogarithmic in the input parameters. We now provide details.

\subsubsection{Step 1: Sparse $\ell_1$ embedding}\label{sec:sparsel1}
\paragraph{Quadtree.} We impose over the point set $\mathcal X \cup \mathcal Y$ a randomly shifted quadtree $T_1$ with $\log(d\Phi_{\mathbf C})+1$ levels. To recap this, let $\Phi>0$ be the smallest power of $2$ such that $\mathcal X \cup \mathcal Y$ is fully enclosed in a $d$-dimensional hypercube $H_0$ with side-length $\Phi$, and assume w.l.o.g.~for simplicity (by applying an appropriate translation to $\mathcal X \cup \mathcal Y$) that $H_0=[0,\Phi]^d$. The quadtree construction starts with the hypercube $H=\sigma+[-\Phi,\Phi]^d$, where $\sigma\in\R^d$ is a random vector with i.i.d.~coordinates chosen uniformly from $[0,\Phi]$. The hypercube $H$ serves as the root of the quadtree, and since $H_0$ encloses all of $\mathcal X \cup \mathcal Y$, it can be easily seen that $H$ too encloses all of $\mathcal X \cup \mathcal Y$.
To generate the next level, $H$ is partitioned into $2^d$ sub-hypercubes by halving it exactly in the middle of each dimension. We add the non-empty sub-hypercubes (i.e., those that contain any points from $\mathcal X \cup \mathcal Y$) as children of $H$ in the quadtree, 
forming the next level. We repeat this process until we generate $\log(d\Phi_{\mathbf C})+1$ levels. Note that the hypercubes in every level of the quadtree induce a partition of $\mathcal X\cup \mathcal Y$, 
or in other words, every point in $\mathcal X\cup \mathcal Y$ is contained in exactly one hypercube in each level. 
The definition of $\Phi_{\mathbf C}$ ensures that the in the final quadtree, the last partition is refined enough so that each every contains either only points from $\mathcal X$ or only points from $\mathcal Y$. 

Let $v$ be a node in the quadtree. Let $\mathrm{side}(v)$ denote the side-length of its associated hypercube (the side-length of the root is $2\Phi$, the side-length of each of its children is $\Phi$, and so on). To every edge connecting $v$ with a child node, we assign the edge weight $\mathrm{side}(v)$. The tree distance between $x\in\mathcal X$ and $y\in\mathcal Y$, denoted $T_1(x,y)$, is defined as the sum of edge weights along the (unique) path in $T_1$ that connects the (unique) leaf that contains $x$ to the (unique) leaf that contains $y$. Recall that $N=|\mathcal X|$ and $M=|\mathcal Y|$, and let us assume w.l.o.g.~$N\geq M$.
The following guarantees of the quadtree is known \citep{andoni2008earth,backurs2020scalable}:
\begin{lemma}\label{lmm:quadtreedistance}
\begin{enumerate}
    \item For every $x\in\mathcal X$ and $y\in\mathcal Y$, $\E[T_1(x,y)]\leq O(\log(d\Phi_{\mathbf C}))\cdot\norm{x-y}_1$.
    \item With probability $0.99$, for all pairs $x\in\mathcal X$ and $y\in\mathcal Y$ simultaneously, $T_1(x,y)\geq (O(\log N))^{-1}\cdot\norm{x-y}_1$.
\end{enumerate}
\end{lemma}

\paragraph{Quadtree OT.} Let $W_{T_1}$ denote the Wasserstein-1 distance on $\mathcal X \cup \mathcal Y$ with the ground metric $T_1(\cdot,\cdot)$. That is, for every measure $\alpha$ on $\mathcal X$ and measure $\beta$ on $\mathcal Y$, $W_{T_1}$ is defined as
\[ W_{T_1}(\alpha,\beta) = \inf_{\pi\in\Pi(\alpha,\beta)}\sum_{(x,y)\in\mathcal X\times\mathcal Y}W_{T_1}(x,y)\cdot\pi(x,y) .\]

\paragraph{Sparse $\ell_1$ embedding.}
We can isometrically embed the quadtree distance and the $W_{T_1}$ distance into $\ell_1$. To this end, let $D$ be the number of nodes in the quadtree $T_1$ excluding its root, and consider the space $\R^D$ with a coordinate per quadtree node excluding its root. The embedding $f:\mathcal X \cup \mathcal Y \rightarrow \R^D$ is defined as follows. Let $z\in\mathcal X\cup\mathcal Y$. For every quadtree node $v$, if $z$ is contained in the hypercube associated with $v$, then we set coordinate $v$ of $f(z)$ to $f(z)_v=\mathrm{side}(v)$. Otherwise, we set $f(z)_v=0$. Note that $f(z)$ is a sparse vector: since $z$ is contained in a single hypercube in each of the $\log(d\Phi_{\mathbf C})$ levels of the quadtree (excluding the top level which contains only the root and induces a trivial partition), $f(z)$ has at most $\log(d\Phi_{\mathbf C})$ non-zero coordinates. Furthermore, it is well-known, and can be easily checked, that $T_1(x,y)=\norm{f(x)-f(y)}_1$ for every $x\in\mathcal X$, $y\in\mathcal Y$. 

Next, let $\alpha$ be a measure on $\mathcal X$. We define $f(\alpha)\in\R^d$ as $f(\alpha)=\sum_{x\in\mathcal X}\alpha(x)\cdot f(x)$. Similarly, for a measure $\beta$ on $\mathcal Y$, we let $f(\beta)=\sum_{y\in\mathcal Y}\beta(y)\cdot f(y)$. Again, it is well-known (e.g., \cite{charikar2002similarity,indyk2003fast}) that
\[ \norm{f(\alpha)-f(\beta)}_1 = W_{T_1}(\alpha,\beta) . \]
Furthermore, if $\alpha$ is supported on $n$ points and $\beta$ is supported on $m$ points, then $f(\alpha)$ has at most $n\log(d\Phi_{\mathbf C})$ non-zero entries, and and $f(\beta)$ has at most $m\log(d\Phi_{\mathbf C})$ non-zero entries.

Putting these together with \Cref{lmm:quadtreedistance}, we have the following: 
\begin{corollary}\label{cor:quadtreedistance}
\begin{enumerate}
    \item For every measure $\alpha$ on $\mathcal X$ and measure $\beta$ on $\mathcal Y$, $\E\norm{f(\alpha)-f(\beta)}_1\leq O(\log(d\Phi_{\mathbf C}))\cdot \mathrm{OT}(\alpha,\beta)$.
    \item With probability $0.99$, for all pairs of measures $\alpha$ on $\mathcal X$ and measure $\beta$ on $\mathcal Y$ simultaneously, $\norm{f(\alpha)-f(\beta)}_1\geq (O(\log N))^{-1}\cdot\mathrm{OT}(\alpha,\beta)$.
\end{enumerate}
\end{corollary}

\subsubsection{Step 2: Flowtree}
\Cref{cor:quadtreedistance} implies in particular that the true OT distances between the batches $\{\alpha^s\}_{s=1}^k$ of $\mathcal X$ and the batches $\{\beta^t\}_{s=1}^k$ of $\mathcal Y$ is now embedded as the $\ell_1$ distance between their corresponding vectors $\{f(\alpha^s)\}_{s=1}^k$ and $\{f(\beta^t)\}_{t=1}^k$ in $\R^D$. On these $\ell_1$ vectors, we can apply the Flowtree algorithm from \cite{backurs2020scalable}. For this, we need a bound on the dimensionality and on the aspect ratio of this $\ell_1$ metric. 

\begin{proposition}[$\ell_1$ embedding dimensionality]\label{clm:flowtreedimension}
We can let $D=O(N\log(d\Phi_{\mathbf C}))$.
\end{proposition}
\begin{proof}
Recall that $N=nk$ and $M=nk$, and that we assume w.l.o.g.~$N\geq M$ and hence $n\geq m$. As noted above, each vector in $\{f(\alpha^s)\}_{s=1}^k \cup \{f(\beta^t)\}_{t=1}^k$ has at most $n\log(d\Phi_{\mathbf C})$ non-zero coordinates, so together they have at most $2kn\log(d\Phi_{\mathbf C})=O(N\log(d\Phi_{\mathbf C}))$ non-zero coordinates.
\end{proof}

\begin{proposition}[$\ell_1$ embedding aspect ratio]\label{clm:flowtreeaspectratio}
Let $\mathbf C_{\max}=\max_{i,j}C_{ij}$ and $\mathbf C_{\min}=\min_{i,j}C_{ij}$.\footnote{Recall that we assume that $\mathbf C_{\min}>0$, and note that $\Phi_{\mathbf C}=\mathbf C_{\max}/\mathbf C_{\min}$.}
Let $\Phi_f$ be defined as
\begin{equation}\label{eq:flowtreeaspectratio}
  \Phi_f = \frac{\max_{s,t\in[k]}\norm{f(\alpha^s)-f(\beta^t)}_1}{\min_{s,t\in[k]}\norm{f(\alpha^s)-f(\beta^t)}_1} .
\end{equation}
Then, with probability $0.98$, $\Phi_f=O(\Phi_{\mathbf C}\cdot\log(d\Phi_{\mathbf C})\cdot k^2\cdot \log N)$. 
\end{proposition}
\begin{proof}
We start by lower-bounding the denominator in \cref{eq:flowtreeaspectratio}:
\[ \min_{s,t\in[k]}\norm{f(\alpha^s)-f(\beta^t)}_1 \geq \frac{\min_{s,t\in[k]}\mathrm{OT}(\alpha^s,\beta^t)}{O(\log N)} \geq \frac{\mathbf C_{\min}}{O(\log N)}, \]
where the first inequality is by \Cref{cor:quadtreedistance}, and the second inequality is since $\mathbf C_{\min}$ is a lower bound on $\norm{x-y}_1$ for every $x\in\mathcal X$ and $y\in\mathcal Y$. We proceed to upper-bounding the numerator in \cref{eq:flowtreeaspectratio}: we have,
\begin{align*}
    \E\left[\max_{s,t\in[k]}\norm{f(\alpha^s)-f(\beta^t)}_1\right] &\leq \E\left[\sum_{s,t\in[k]}\norm{f(\alpha^s)-f(\beta^t)}_1\right] & \\
    &= \sum_{s,t\in[k]}\E\left[\norm{f(\alpha^s)-f(\beta^t)}_1\right] & \\
    &\leq O(\log(d\Phi_{\mathbf C}))\cdot \sum_{s,t\in[k]}\mathrm{OT}(\alpha,\beta) & \text{by \Cref{cor:quadtreedistance}} \\
    &\leq O(\log(d\Phi_{\mathbf C}))\cdot k^2 \cdot \max_{s,t\in[k]}\mathrm{OT}(\alpha,\beta) \\
    &\leq O(\log(d\Phi_{\mathbf C}))\cdot k^2 \cdot \mathbf C_{\max},
\end{align*}
where the last inequality is since $\mathbf C_{\max}$ is an upper bound on $\norm{x-y}_1$ for every $x\in\mathcal X$ and $y\in\mathcal Y$. 
By Markov's inequality, we have with probability $0.99$ that 
\[ \max_{s,t\in[k]}\norm{f(\alpha^s)-f(\beta^t)}_1 \leq 100\cdot O(\log(d\Phi_{\mathbf C}))\cdot k^2 \cdot \mathbf C_{\max}. \]
Taking a union bound over this event with the second item of \Cref{cor:quadtreedistance}, both hold simultaneously with probability $0.98$.  The claim follows by combining the lower bound on the denominator and the upper bound on the numerator of \cref{eq:flowtreeaspectratio}, and recalling that $\Phi_{\mathbf C}=\mathbf C_{\max}/\mathbf C_{\min}$.
\end{proof}

Now we can complete the proof of the approximation guarantee of BHOT-Tree. 
To this end, we consider three matchings between the batches $\{\alpha^s\}_{s=1}^k$ and $\{\beta^t\}_{t=1}^k$:
\begin{itemize}
    \item Let $\pi^*:[k]\rightarrow[k]$ be the true optimal matching that induces BHOT, that is,
    \[ \pi^* = \mathrm{arginf}_{\pi:[k]\rightarrow[k]}\sum_{s=1}^k\mathrm{OT}(\alpha^s,\beta^{\pi(s)}). \]
    Note that
    \[ \mathrm{BHOT} = \inf_{\pi:[k]\rightarrow[k]}\sum_{s=1}^k\mathrm{OT}(\alpha^s,\beta^{\pi(s)}) =\sum_{s=1}^k\mathrm{OT}(\alpha^s,\beta^{\pi^*(s)}). \]
    \item Let $\pi_f:[k]\rightarrow[k]$ be the optimal matching between the batches under their $\ell_1$ embedding through $f$, that is,
    \[ \pi_f = \mathrm{arginf}_{\pi:[k]\rightarrow[k]}\sum_{s=1}^k\norm{f(\alpha^s)-f(\beta^{\pi(s)})}_1. \]
    \item Let $\tilde\pi:[k]\rightarrow[k]$ be the matching returned by Flowtree between the $\ell_1$-embedded batches $\{f(\alpha^s)\}_{s=1}^k$ and $\{f(\beta^t)\}_{t=1}^k$.
\end{itemize}
Our goal is now to prove that the matching returned by Flowtree is approximately optimal for BHOT, and in particular, that with probability 0.95 the following holds:
\begin{equation}\label{eq:flowtreegoal}
    \mathrm{BHOT} \leq \sum_{s=1}^k\mathrm{OT}(\alpha^s,\beta^{\tilde\pi(s)}) \leq \tilde O(\log^2(N)\cdot\log^2(d\Phi_{\mathbf C})\cdot\log k)\cdot\mathrm{BHOT}.
\end{equation}
The left inequality is immediate by the optimality of $\pi^*$ for BHOT. We now show the right inequality. By the second item in \Cref{cor:quadtreedistance}, with probability 0.99 we have
\[
  \sum_{s=1}^k\mathrm{OT}(\alpha^s,\beta^{\tilde\pi(s)}) \leq O(\log N)\cdot\sum_{s=1}^k\norm{f(\alpha^s)-f(\beta^{\tilde\pi(s)})}_1.
\]
Flowtree guarantees \citep{backurs2020scalable} that for the $\ell_1$-embedded batches, with probability 0.99, the returned matching $\tilde\pi$ satisfies
\[
 \sum_{s=1}^k\norm{f(\alpha^s)-f(\beta^{\tilde\pi(s)})}_1 \leq O(\log(k)\cdot\log(D\Phi_f))\cdot \sum_{s=1}^k\norm{f(\alpha^s)-f(\beta^{\pi_f(s)})}_1. 
\]
The optimality of $\pi_f$ for the $\ell_1$-embedded batches implies that 
\[
  \sum_{s=1}^k\norm{f(\alpha^s)-f(\beta^{\pi_f(s)})}_1 \leq \sum_{s=1}^k\norm{f(\alpha^s)-f(\beta^{\pi^*(s)})}_1. 
\]
By the first item in \Cref{cor:quadtreedistance} and linearity of expectation, we have
\[
 \E\left[\sum_{s=1}^k\norm{f(\alpha^s)-f(\beta^{\pi^*(s)})}_1\right] \leq O(\log(d\Phi_{\mathbf C})\cdot \sum_{s=1}^k\mathrm{OT}(\alpha^s,\beta^{\pi^*(s)}) = O(\log(d\Phi_{\mathbf C})\cdot\mathrm{BHOT}
\]
and therefore by Markov's inequality, with probability 0.99 we have 
\[
 \sum_{s=1}^k\norm{f(\alpha^s)-f(\beta^{\pi^*(s)})}_1\leq 100\cdot O(\log(d\Phi_{\mathbf C})\cdot\mathrm{BHOT}.
\]
Taking a union bound over the two mentioned events and concatenating the inequalities, we get that with probability 0.98,
\[ 
  \sum_{s=1}^k\mathrm{OT}(\alpha^s,\beta^{\tilde\pi(s)}) \leq
  O(\log(N)\cdot\log(k)\cdot\log(D\Phi_f)\cdot\log(d\Phi_{\mathbf C})) \cdot \mathrm{BHOT}.
\]
By \Cref{clm:flowtreedimension,clm:flowtreeaspectratio} we have $D=O(N\log(d\Phi_{\mathbf C})$ and $\Phi_f=O(\Phi_{\mathbf C}\cdot\log(d\Phi_{\mathbf C})\cdot k^2\cdot \log N)$, the latter holding with probability $0.98$. Taking a union bound over this event too and plugging $D$ and $\Phi_f$ above, we obtain \cref{eq:flowtreegoal} as desired.

\paragraph{Computational efficiency.}
Computing the embedding $f(\alpha)$ for every batch $\alpha$ can be parallelized across the batches. For every batch, this takes time $O(nd\log(d\Phi_{\mathbf C}))$. Then, the Flowtree algorithm embeds each batch into another quadtree in time $O(D\log\Phi_f)$ per batch, and then computes the approximate matching in time $O(k\log\Phi_f)$. Plugging \Cref{clm:flowtreedimension,clm:flowtreeaspectratio} for $D$ and $\Phi_f$, the overall running time is $\tilde O(Nd\log(d\Phi_{\mathbf C}))$.

\subsection{Proof of Theorem~\ref{thm:stars}}
The algorithm is similar to BHOT-Tree from the previous section: it begins with the same sparse $\ell_1$ embedding step from \Cref{sec:sparsel1}, and then embeds the $\ell_1$-embedded batches into another quadtree $T_2$ with $\Phi_f$ levels. The difference is that instead of computing the optimal matching in the tree metric (like Flowtree), we do the following: in every node $v$ of $T_2$, we choose an arbitrary point $x_v$ in the hybercube associated with $v$, and draw an edge between $x_v$ and every other point in $v$ (thus adding a star graph on the points in $v$, cenetered at $x_v$). \cite{har2013euclidean} showed that repeating this $\tilde O(k^\rho)$ times yields a graph $G$ whose shortest path distance $G(\cdot,\cdot)$ satisfies the following with high probability: for every pair of batches $\alpha^s,\beta^t$,
\[ \norm{f(\alpha^s)-f(\beta^t)}_1 \leq G(\alpha^s,\beta^t) \leq O(\rho^{-1})\cdot\norm{f(\alpha^s)-f(\beta^t)}_1. \]
\begin{remark}
The statement in \cite{har2013euclidean} has a better approximation factor $O(\rho^{-1/2})$ instead of the $O(\rho^{-1})$ stated above.
The reason is that \cite{har2013euclidean} worked with Euclidean metrics, while here we work with an $\ell_1$ metric. The reason for the different approximation factor is that the best locality-sensitive hashing (LSH, \cite{indyk1009approximate}) dependence for Euclidean metrics is better than for $\ell_1$, namely, $\tilde O(n^{\rho})$ repetitions yield a $O(\rho^{-1/2})$-approximation in Euclidean metrics but an $O(\rho^{-1})$-approximation for $\ell_1$ metrics \citep{andoni2006near}. The reason we work here with an $\ell_1$ metric is that we can approximately embed the OT distance into $\ell_1$ (as done in \Cref{sec:sparsel1}) but not into $\ell_2$. Applying the LSH-based technique of \cite{har2013euclidean} to $\ell_1$ instead of $\ell_2$ yields the bound stated above.
\end{remark}

Consequently, by arguments similar to the ones in the previous section, the optimal matching computed on $G$ is an $O(\rho^{-1})$-approximation for the optimal matching for the embedded batches through $f$, which is in turn an $O(\log(d\Phi_{\mathbf C})\cdot\log N)$-approximation for BHOT. Overall, the stars algorithm returns a $O(\rho^{-1}\log(d\Phi_{\mathbf C})\cdot\log N)$-approximation for BHOT. 

\paragraph{Computational efficiency.}
As in the previous section, computing the $\ell_1$ embedding per batch takes time $O(nd\log(d\Phi_{\mathbf C}))$, and embedding them into the second quadtree $T_2$ takes time $O(D\log\Phi_f)$ per batch. 
The remaining running time is dominated by computing the true OT between every pair of batches connected with an edge in $G$. Since we add (up to) $k$ edges in each level of $T_2$, which has $\log\Phi)_f$ levels, and we repeat this $\tilde O(k^\rho)$ times, the running time of this part is $\tilde O(k^{1+\rho}\log\Phi_f)\cdot\mathrm{ot}(n,m)$.

\section{Details on Bounding Methods}\label{sec:further_methods}

\paragraph{Greedy matching with missing costs.} Given a budget of $B \geq k$, we first solve exactly all diagonal problems $\OT(\alpha^s, \beta^s), s\in \llbracket k \rrbracket$. Then, for as long as there is budget remaining, we find among the rows $\D_{s,:}$ and columns $\D_{:,t}$ with at least one missing entry, the pair $(s,t)$ with highest OT cost. We randomly sample a column $t'$ from among those with missing entries in $\D_{s,:}$, and solve $\OT(\alpha^s, \beta^{t'})$ to fill it. Budget permitting, we do the same for the entries of $\D_{:,t}$, solving $\OT(\alpha^{s'}, \beta^t)$ to fill it. We continue until the entire budget has been consumed. Then, we solve the meta-OT problem with missing costs as with \textsc{bhot}-Missing, described in Section~\ref{sec:tradeoffs}. 

\begin{algorithm}[t]
\caption{Proxy Distance OT Upper Bound}\label{algo:bhot_ub_proxy}
\KwIn{Data $\X \in \R^{N\times d}, \Y\in \R^{M\times d}$;  num.~batches $k$;  }
\tcc{Stage 1: approximating batch OT problems }
$n,m \gets \textup{GetBatchSize}(k,N,M)$\;
$\{ \X^1, \dots, \X^k \} \gets \text{Partition}(\X, n)$\;
$\{ \Y^1, \dots, \Y^k \} \gets \text{Partition}(\Y, m)$\; 
$\tilde{\mathbf{D}} \gets  \mathbf{0}_{k\times k}$ \;
\For{$(s,t) \in \llbracket k \rrbracket \times \llbracket k \rrbracket $}{
    $\mathbf{C}^{st} \gets \textup{PairwiseDistances}(\X^s, \Y^t)$\;
    $\tilde{\mathbf{D}}[s,t] \gets \textup{OTProxy}(\mathbf{C}^{st})$\;
}
$d, \Gamma \gets \OT(\tilde{\mathbf{D}}, \tfrac{1}{k}\ones_k, \tfrac{1}{k}\ones_k)$\;
\tcc{Stage 2: solving select batch problems}
$M \gets \textup{MatchingFromCoupling}(\Gamma)$  \tcp*{$|M|=k$}
$D \gets \emptyset$ \;
\For{$(s,t) \in M$}{
    $D \gets D \cup \OT(\mathbf{C}^{st}, \tfrac{1}{n}\ones_n, \tfrac{1}{m}\ones_m)$\;
}
\tcc{Stage 3: aggregating batch solutions}
$d \gets \textup{Mean}(D)$\;
\Return{d}
\end{algorithm}

\section{Economic Interpretation of Batch-Hierarchical Problem}
We revisit the bakery-caf\'e analogy by \citet{villani2008optimal} to provide an economic interpretation of the batch-hierarchical OT problem. Consider a large number of bakeries (located at $x_i$), producing loaves of bread, that need to be transported every morning to caf\'es (at locations $y_j$) around a city, say Paris. The baker-caf\'e consortium is tasked with finding a transference plan which determines the amount of bread to be transferred from each bakery to each caf\'e across the city, at minimal cost (e.g., some function of the distance traveled $c(x_i,y_j)$). This corresponds to the primal optimal transport problem. If, instead, the transportation of the loaves were to be outsourced to a transportation company, it would solve the dual OT problem. Indeed, the objective of this company would be to set their price schemes $\f,\g$ (pick-up and delivery prices per unit/location, respectively) so as to maximize their profit (total pick-up and delivery revenue), while staying competitive, i.e., not charging more that the cost of transportation incurred by the producers: $\f_i + \g_j \leq c(x_i, y_j)$. 

When the number of bakeries and caf\'es is too large (as is undoubtedly the case for Paris), the baker-caf\'e consortium might not be able to solve a problem of that scale. Thus, instead of relying on a single central planning authority, the consortium might organize into neighborhood sub-units, one for each \textit{arrondissement} in Paris. For each pair of neighborhoods, their corresponding baker and caf\'e consortiums could solve the transportation problem between their establishments (problems of the form \plaineqref{eq:primal_prob_batch}). Provided with a collection of these local transfer solutions, the central planning authority proceeds to rescale and aggregate them to ensure all supply and demand constraints are met
(Theorem~\ref{thm:blockwise_linear_primal}), while minimizing the total cost of the operation (i.e., solving problem \plaineqref{eq:primal_prob_meta}). Being limited to choose only from among aggregated re-scaled local solutions, the central authority will in general find a sub-optimal solution (corresponding to the \textsc{bhot} upper bound) compared to the (exact) solution that could be obtained by solving the entire problem at once. 

Suppose now that not all local problems can be solved, e.g., because the bakers in one neighborhood and the caf\'e owners in another one refuse to cooperate, or because labor laws limit how many hours they must devote to solving such problems. In that case, the central planning authority could, given a budget of $k$ sub-problems to solve, still come up with a global (albeit even less optimal) solution using only this $k$ local solutions, e.g., utilizing any of the budget-constrained methods discussed in Section~\ref{sec:ub}. The extreme version of this would be to pair neighborhoods at random, and only transport bread between these pairs according to their locally-optimal solutions (corresponding to the naive mini-batch averaging bound).

\section{Further Experimental Details and Results}\label{sec:exp_details}

\begin{figure}
     \centering
     \begin{subfigure}[b]{0.3\textwidth}
         \centering
         \includegraphics[width=\textwidth]{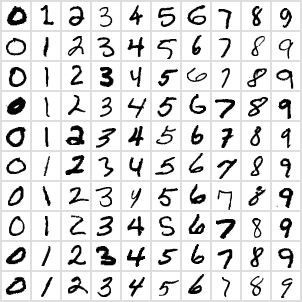}
         \caption{Rotation angle $\theta=0^\circ$}
         \label{fig:mnist_original}
     \end{subfigure}
     \hfill
     \begin{subfigure}[b]{0.3\textwidth}
         \centering
         \includegraphics[width=\textwidth]{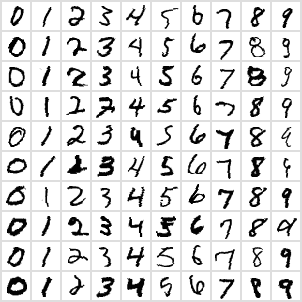}
         \caption{Rotation angle $\theta=4^\circ$}
         \label{fig:mnist_rot4}
     \end{subfigure}
     \hfill
     \begin{subfigure}[b]{0.3\textwidth}
         \centering
         \includegraphics[width=\textwidth]{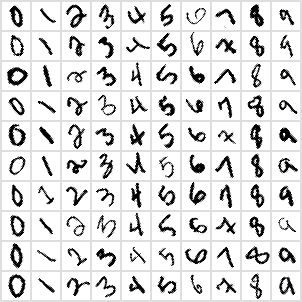}
         \caption{Rotation angle $\theta=45^\circ$}
         \label{fig:mnist_rot45}
     \end{subfigure}
        \caption{\textbf{MNIST Rotation datasets used for two-sample test experiments}. We take the original MNIST dataset (left) and apply rotations of increasing degree (e.g., $4^\circ, 45^\circ$ in center, right) to simulate distributional drifts, which we then try to detect using the various OT bounding methods proposed here as statistics within a two-sample permutation test analysis (Section~\ref{sec:exp_twosample_testing}). Some of our methods, including the best-in-class \textsc{bhot} bound but also the faster budget-constrained variants \textsc{bhot}-Missing and \textsc{bhot}-Tree, are able to reliably and significantly detect drifts with rotations as small as $\pm\theta=4^{\circ}$.
        }
    \label{fig:rotations}
\end{figure}

In Section~\ref{sec:exp_twosample_testing}, we evaluate our methods in a task of detecting distributions shifts on MNIST data. The data generation procedure is described in the main text (see also Figure~\ref{fig:rotations}). To turn our various OT estimates into statistical tests with significance quantification, we treat the problem as a two-sample test, whereby we use the distance statistic to decide whether to reject the null $H_0: \alpha = \beta$ in favor the the alternative $H_1: \alpha \neq \beta$. For this, we use the permutation tests, which have the advantage of being fully non-parametric and making no distributional assumption on the test statistic \citep{good2013permutation,ilmun2022minimax}. In practice, this involves repeatedly mixing and shuffling the datasets, comparing the distance between the original datasets and the random splits of these mixed datasets. For this, we use the function $\texttt{scipy.stats.permutation\_test}$ from the SciPy package \citep{2020SciPy-NMeth} with parameters $\texttt{alternative=`greater'}$ (i.e., a one-sided test) and $\texttt{n\_resamples=200}$ (the number of random mixing repetitions used to obtain p-values). This gives us a single scalar (the p-value for the alternative) for every [method, original dataset, rotated dataset] triplet. We repeat this entire process 5 times using different random slices of the two datasets to obtain confidence intervals and error bars in Figures~\ref{fig:alpha_power} and \ref{fig:power_angle}. Our implementation uses the \hyperlink{https://pythonot.github.io/}{POT} toolbox to solve OT problems, and will be made available upon acceptance.

%

\end{document}